\documentclass{article}

\PassOptionsToPackage{numbers, compress}{natbib}

\usepackage[main, final]{neurips_2026}

\usepackage[utf8]{inputenc} % allow utf-8 input
\usepackage[T1]{fontenc}    % use 8-bit T1 fonts
\usepackage{hyperref}       % hyperlinks
\usepackage{url}            % simple URL typesetting
\usepackage{booktabs}       % professional-quality tables
\usepackage{amsfonts}       % blackboard math symbols
\usepackage{nicefrac}       % compact symbols for 1/2, etc.
\usepackage{microtype}      % microtypography
\usepackage{xcolor}  
\usepackage{amsmath}
\usepackage{graphicx}
\usepackage{multirow}
\usepackage{caption}
\usepackage{amsthm}
\usepackage{graphicx}
\usepackage{wrapfig}
\usepackage{algorithm}
\usepackage{algpseudocode}
\usepackage{titletoc}
\newtheorem{theorem}{Theorem}[section]

\newtheorem{proposition}[theorem]{Proposition}

\usepackage{algorithm}
\usepackage{algpseudocode}

\title{VTV-FM: Flow Matching through Variational Terminal-Velocity Closure}
\author{%
\begin{tabular}{ccc}
\textbf{Haoyang Jiang}$^{*}$ &
\textbf{Yuheng Li}$^{*}$ &
\textbf{Di Yang}$^{*}$ \\[3pt]
Yanhai Xiong &
Haipeng Chen$^{\dagger}$ &
Yi He$^{\dagger}$
\end{tabular}
\\[7pt]
Department of Data Science, William \& Mary\\
Williamsburg, VA 23185, USA\\[4pt]
{\small
\texttt{\{hjiang16,yli95,dyang06,yxiong05,hchen23,yihe\}@wm.edu}
}\\[3pt]
{\small
$^{*}$Equal contribution. \qquad
$^{\dagger}$Corresponding authors.
}
}

\begin{document}

\maketitle

\begin{abstract}
Flow matching (FM) learns generative transport by fitting continuous-time motion from a simple source distribution to the data distribution. 
Most existing methods use first-order bridges: once a source and a target sample are paired, the path is a straight motion with constant velocity. 
FM with optimal transport (OT) improves the pairing, but the bridge itself remains linear, limiting its ability to model curved motion, acceleration, and changing directions. 
A natural remedy is to use second-order phase-space dynamics; however, learning the bridge requires target-side terminal-velocity
information that static datasets do not provide.
We propose \textbf{V}ariational \textbf{T}erminal-\textbf{V}elocity \textbf{F}low \textbf{M}atching (\textbf{VTV-FM}), a second-order FM framework that derives the missing velocity by minimizing acceleration energy, yielding a closed-form closure for static data. 
The same minimum-acceleration variational construction also defines the OT pairing cost and the acceleration targets used for training.
% When terminal velocity is available, the method directly incorporates it through the full second-order bridge. 
%When terminal velocity is available, as in PDE data or temporal sequences, the method directly incorporates it through the full second-order bridge. 
Experiments on low-dimensional datasets, PDE-governed physical fields, and CIFAR-10 show that VTV-FM improves transport geometry and generation quality over first-order and high-order FM baselines.
Our code is available at \url{https://github.com/HaoyangJiang-WM/VTV-FM}.
\end{abstract}

\vspace{-0.55em}
\section{Introduction}
\vspace{-0.5em}
\label{sec:introduction}

Flow Matching (FM) has emerged as a flexible generative model by learning continuous-time vector fields along prescribed probability paths~\citep{lipman2023flowmatching,albergo2023stochastic}. 
A common choice is the first-order linear bridge: given a paired source sample and target sample, the model learns a straight path from \(x_0\) to \(x_1\), corresponding to constant-velocity transport~\citep{lipman2023flowmatching}. 
OT-based Flow Matching further improves this framework by replacing independent source--target pairing with an optimal coupling, thereby reducing transport cost and vector-field conflicts~\citep{tong2024minibatchot,pooladian2023multisample}. 
However, even after the pairing is improved, the conditional bridge itself remains a first-order linear interpolant. 
For data distributions with curvature, branching structures, or separated modes, straight paths may pass through low-density regions and require the model to learn vector fields in poorly supported intermediate regions. 
Meanwhile, trajectory intersections can induce conflicting velocity supervision, forcing the learned vector field to average incompatible directions~\citep{liu2023flowstraightfast}. 
% This motivates a second-order phase-space formulation, where transport trajectories can bend, accelerate, and change direction rather than being restricted to linear motion.
This motivates a second-order phase-space formulation, where transport
can exhibit nonuniform speed and curved position paths with time-varying
tangent directions, rather than being restricted to constant-velocity
linear interpolation.

Constructing a second-order bridge is fundamentally more challenging than constructing a first-order path. 
A first-order bridge only requires endpoint positions, whereas a second-order phase-space bridge also requires endpoint velocities, since these boundary velocities determine the intermediate trajectory and the induced acceleration supervision. 
The initial velocity \(v_0\) can be sampled from a source phase-space prior, but static datasets typically provide only the terminal position \(x_1\), leaving the terminal velocity \(v_1\) unspecified. 
Related high-order and acceleration-aware FM methods enrich path geometry
through higher-order supervision, learned path constructions, boundary
constraints, or acceleration-based refinement mechanisms~
\citep{chen2025highordermatching,chen2025nrflow,
su2025highorderflowmatching,luo2025curveflow,
hu2025boundaryconditions,yue2025oatfm}.

VTV-FM instead formulates the static-data case by treating \(v_1\) as a
latent boundary variable and eliminating it variationally, tying the
terminal-velocity completion to the bridge geometry and pairing cost.
This issue is critical because \(v_1\) controls the full phase-space bridge. 
Externally specified terminal velocities can therefore introduce unnecessary curvature or acceleration, biasing the regression targets. 
Figure~\ref{fig:v1_closure} illustrates this sensitivity: under the same observable boundary information \((x_0,v_0,x_1)\), different choices of \(v_1\) induce different position trajectories, velocity profiles, and acceleration-energy costs. 
Since deterministic FM sampling relies on numerical integration of the learned dynamics at inference time, such trajectory distortions can make discretization errors accumulate more easily and increase sensitivity to solver choice and time discretization~\citep{liu2023flowstraightfast,shaul2024bespoke}.

\begin{figure}[t]
    \centering
\includegraphics[width=0.82\linewidth]{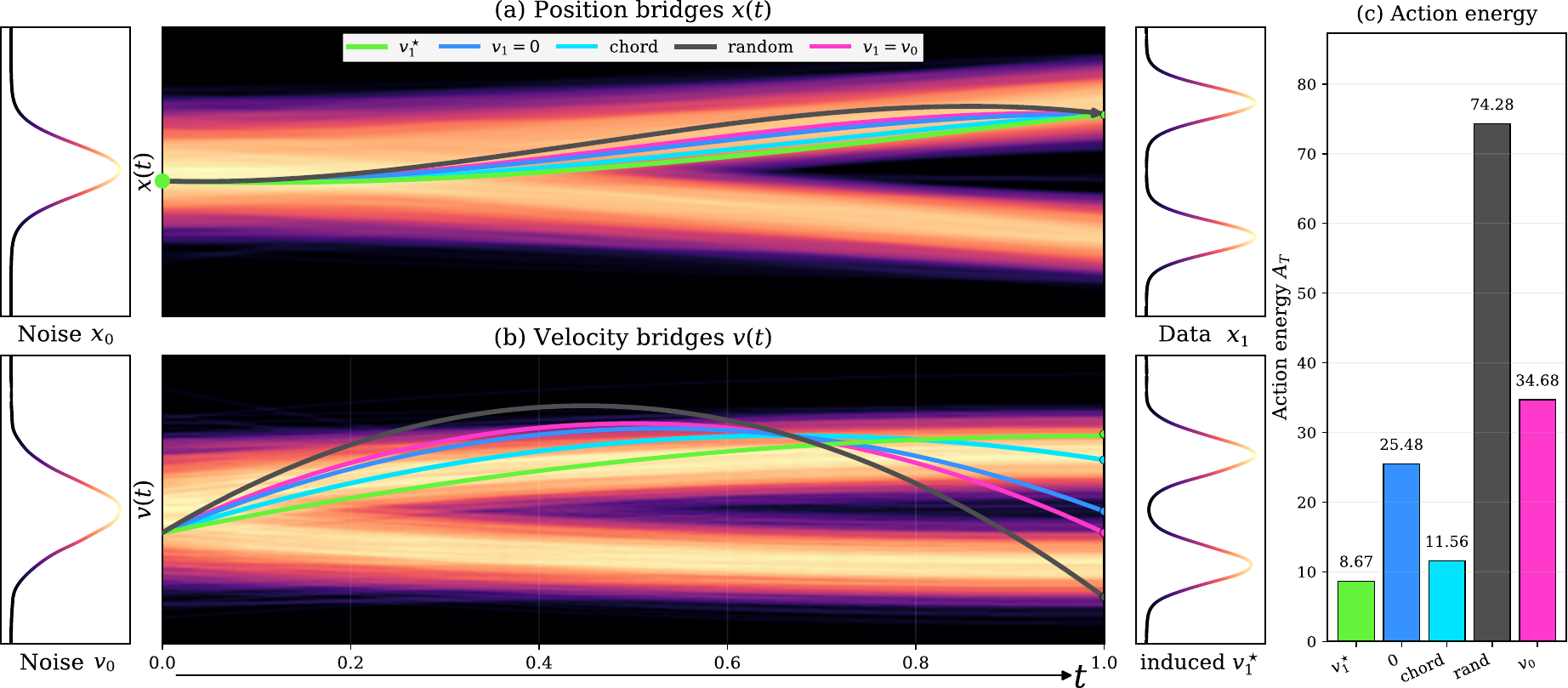}
\caption{
Effect of terminal-velocity choices on second-order bridge supervision.
All curves in (a,b) use the same observable boundary information \((x_0,v_0,x_1)\), but different terminal velocities \(v_1\).
In (a,b), the horizontal axis is normalized time \(t/T\), and colors denote the same terminal-velocity choices: \(v_1^\star\), \(v_1=0\), chord velocity \((x_1-x_0)/T\), random velocity, and \(v_1=v_0\).
\textbf{(a)} The position bridges show the evolution of \(x(t)\) under different terminal-velocity choices.
\textbf{(b)} The velocity bridges show the corresponding evolution of \(v(t)\).
\textbf{(c)} The bars show the acceleration-energy cost for each terminal-velocity choice.
Lower cost corresponds to smoother supervision, while externally specified terminal velocities may induce biased supervision.
}
    \vspace{-2.0em}
    \label{fig:v1_closure}
\end{figure}

To address these structural limitations, we propose \textbf{V}ariational \textbf{T}erminal-\textbf{V}elocity \textbf{F}low \textbf{M}atching (\textbf{VTV-FM}). 
Its core is a minimum-acceleration variational principle: the same variational problem simultaneously induces the second-order bridge, the terminal-velocity closure, and the OT pairing cost. 
For a complete boundary state \((x_0,v_0,x_1,v_1)\), the minimum-acceleration problem yields a cubic Hermite bridge, together with the corresponding intermediate velocity and acceleration supervision. 
The optimal action value of this trajectory further defines a full-boundary second-order transport cost.

In static-data regimes where the terminal velocity is missing, VTV-FM does not introduce an auxiliary estimator or artificial constraint. 
Instead, it treats \(v_1\) as a latent boundary variable and directly minimizes the same full-boundary cost with respect to \(v_1\). 
This optimization has a closed-form solution, yielding the variational terminal-velocity closure \(v_1^\star\). 
Substituting \(v_1^\star\) back into the original cost gives a projected OT cost that depends only on the observable boundary variables \((x_0,v_0,x_1)\). 
In this way, the missing \(v_1\) in static data is determined intrinsically by the same minimum-acceleration principle, rather than being specified by an external mechanism. 
When the terminal velocity is observed or derivable, as in PDE data or temporal sequences, VTV-FM directly uses the full-boundary bridge without applying the closure. 
Thus, the terminal-velocity closure, the OT pairing cost, and the acceleration supervision induced by the second-order bridge are all determined by the same minimum-acceleration variational principle, rather than being specified separately by external rules.

Our contributions include: 
{(1)} a phase-space formulation of FM that identifies missing terminal velocity as the key obstacle for second-order bridges on static data; 
{(2)} together with the proposed VTV-FM framework, a closed-form terminal-velocity closure derived from a minimum-acceleration principle, which also defines the bridge and projected OT pairing cost; and 
{(3)} empirical validation on low-dimensional transport, PDE-governed fields, and CIFAR-10, where VTV-FM improves over first-order, OT-based, and high-order FM baselines.

\vspace{-0.25em}
\section{Preliminaries}
\label{sec:preliminaries}
\vspace{-0.25em}
\textbf{First-order FM and OT pairing.}
Let \(\rho_0\) and \(\rho_1\) denote the source and target distributions in position space. 
Given \(x_0\sim\rho_0\) and \(x_1\sim\rho_1\), standard FM constructs the linear bridge \(x_\tau=(1-\tau)x_0+\tau x_1\), \(\tau\in[0,1]\), with conditional velocity \(u_\tau=x_1-x_0\). 
The velocity field is trained by \(\mathcal L_{\mathrm{FM}}(\theta)=\mathbb E\|v_\theta(x_\tau,\tau)-u_\tau\|_2^2\).
OT-CFM ~\citep{tong2024minibatchot} replaces the independent source--target pairing with an OT coupling. 
Given minibatch samples, it constructs a cost matrix using the quadratic cost \(c_{\mathrm{OT}}(x_0,x_1)=\|x_1-x_0\|_2^2\), and approximately solves the coupling by entropic OT / Sinkhorn iterations~\citep{cuturi2013sinkhorn,tong2024minibatchot}. 
From a variational view, this cost corresponds to the minimum value of the first-order kinetic action \(\int_0^1\|\dot x_t\|_2^2dt\) between fixed endpoints, whose minimizer is a constant-velocity straight line. 
Thus, OT-CFM improves endpoint pairing, but the bridge remains linear and first-order.

\textbf{Second-order phase-space dynamics.}
To construct second-order transport, we lift the state to phase space, \(z_t=(x_t,v_t) \in\mathcal Z:=\mathcal X\times\mathcal V\), where \(\mathcal X=\mathbb R^d\) and 
\(\mathcal V=\mathbb R^d\) are the state and velocity space, respectively. We consider second-order dynamics \(v_t = \dot x_t\), \(a_t = \dot v_t = \ddot x_t\). 
Unlike a first-order bridge, which only requires endpoint positions \((x_0,x_1)\), a second-order bridge typically requires the full phase-space boundary state \((x_0,v_0,x_1,v_1)\). This phase-space formulation can be viewed as a second-order analogue of first-order OT: we use the acceleration action \(\mathcal A_T[x]=\int_0^T\|\ddot x_t\|_2^2dt\) to induce a second-order bridge and the corresponding endpoint cost. 
Given this cost, the pairing step remains a discrete OT problem and can be solved with minibatch Sinkhorn as in OT-CFM. The key difference is that static datasets usually provide only the terminal position \(x_1\), but not the terminal velocity \(v_1\). 
Therefore, determining the missing \(v_1\) is the central obstacle in extending FM to second-order phase-space bridges. 
The next section shows how the acceleration action yields the bridge, terminal-velocity closure, and OT pairing cost.

\section{Proposed Method: VTV-FM}
\label{sec:method}
% \vspace{-0.25em}

We present VTV-FM, a second-order Flow Matching framework in phase space. Its central idea is to derive the transport bridge, terminal-velocity closure, and OT pairing cost from a single minimum-acceleration variational principle, rather than specifying them as separate design choices. This unified construction allows VTV-FM to use the full-boundary bridge when terminal velocities are available, and to close the missing terminal velocity variationally when they are not. We first derive the closed-form full-boundary bridge and cost for
\((x_0,v_0,x_1,v_1)\) (Section~\ref{sec:full_boundary_bridge}). We then project this cost onto \((x_0,v_0,x_1)\), obtaining a closed-form terminal-velocity closure and projected OT cost for the latent-\(v_1\) regime
(Section~\ref{sec:terminal_velocity_projection}). Finally, we use these quantities for minibatch pairing, trajectory supervision, acceleration matching, and sampling (Sections~\ref{sec:minibatch_pairing}--\ref{sec:acceleration_matching}).

\subsection{Full-Boundary Minimum-Acceleration Bridge}
\label{sec:full_boundary_bridge}
\vspace{-0.25em}
We consider the idealized full-boundary setting, where a complete boundary state
\((x_0,v_0,x_1,v_1)\) and a transport horizon \(T>0\) are given. Let the set of admissible trajectories be
\[
\Gamma_T(x_0,v_0,x_1,v_1)
=
\left\{
\gamma\in H^2([0,T];\mathbb R^d):
\gamma(0)=x_0,\ 
\dot\gamma(0)=v_0,\ 
\gamma(T)=x_1,\ 
\dot\gamma(T)=v_1
\right\},
\]
which satisfy the full phase-space boundary
conditions. Following the second-order OT viewpoint, each trajectory is measured
by the squared-acceleration action
\(\mathcal A_T[\gamma]=\int_0^T\|\ddot\gamma(t)\|_2^2\,dt\). The full-boundary transport cost is the minimum action value, and the
corresponding variational bridge is any action minimizer, which are defined as follows:
\begin{subequations}\label{eqn_var_cost_bridge}
\begin{align}
C_T(x_0,v_0,x_1,v_1)
&=
\inf_{\gamma\in\Gamma_T(x_0,v_0,x_1,v_1)}
\mathcal A_T[\gamma],
\label{eqn_var_cost}
\\
\gamma_T^\star(x_0,v_0,x_1,v_1)
&\in
\arg\min_{\gamma\in\Gamma_T(x_0,v_0,x_1,v_1)}
\mathcal A_T[\gamma].
\label{eqn_var_bridge}
\end{align}
\end{subequations}
The bridge induces the phase-space trajectory and acceleration target:
\[
x_t^\star=\gamma_T^\star(t),\qquad
v_t^\star=\dot\gamma_T^\star(t),\qquad
a_t^\star=\ddot\gamma_T^\star(t).
\]
Equations (\ref{eqn_var_cost}) and (\ref{eqn_var_bridge}) establish the basic variational objective used by VTV-FM. 
For a complete boundary state, the endpoint cost, intermediate trajectory, and acceleration target are all induced by the same minimum-acceleration principle. While these two equations specify the bridge and cost implicitly, Proposition \ref{prop:full_boundary_bridge} provides their closed-form solutions as functions of the boundary states $(x_0, v_0, x_1, v_1)$ and shows that the cubic Hermite bridge is not an additional path assumption, but the unique solution of the full-boundary minimum-acceleration problem.
This cubic Hermite construction is classical in spline theory~\citep{deboor1978splines}.
% \subsection{Closed-Form Bridge and Full-Boundary Cost}
% \label{sec:closed_form_bridge}

%We now characterize the variational bridge in closed form. 
% The result shows that the cubic Hermite bridge is not an additional path assumption, but the unique solution of the full-boundary minimum-acceleration problem.

\begin{proposition}[Full-boundary minimum-acceleration bridge]
\label{prop:full_boundary_bridge}
For any \(T>0\) and complete boundary state \((x_0,v_0,x_1,v_1)\), the variational bridge in Eq.~(\ref{eqn_var_bridge}) exists and is unique. 
Let \(\tau=t/T\). The unique minimizer is characterized by the cubic Hermite curve
\begin{equation}\label{eqn_hermite_bridge_x}
    x_t^\star =
h_{00}(\tau)x_0
+
h_{10}(\tau)T v_0
+
h_{01}(\tau)x_1
+
h_{11}(\tau)T v_1,
\end{equation}

where
\[
h_{00}(\tau)=2\tau^3-3\tau^2+1,\qquad
h_{10}(\tau)=\tau^3-2\tau^2+\tau,
\]
\[
h_{01}(\tau)=-2\tau^3+3\tau^2,\qquad
h_{11}(\tau)=\tau^3-\tau^2.
\]
The induced velocity and acceleration are
\begin{equation}\label{eqn_hermite_bridge_v}
    \begin{aligned}
        v_t^\star =
\frac{1}{T}
\left[
h_{00}'(\tau)x_0
+
h_{10}'(\tau)T v_0
+
h_{01}'(\tau)x_1
+
h_{11}'(\tau)T v_1
\right], \\
a_t^\star
=
\frac{1}{T^2}
\left[
h_{00}''(\tau)x_0
+
h_{10}''(\tau)T v_0
+
h_{01}''(\tau)x_1
+
h_{11}''(\tau)T v_1
\right].
    \end{aligned}
\end{equation}
Moreover, the full-boundary cost in Eq. (\ref{eqn_var_cost}) admits the closed form
\begin{equation}\label{eqn_close_cost}
    C_T(x_0,v_0,x_1,v_1)
=
\frac{1}{T}
\left[
12
\left\|
\frac{x_1-x_0}{T}
-
\frac{v_0+v_1}{2}
\right\|_2^2
+
\|v_1-v_0\|_2^2
\right].
\end{equation}
\end{proposition}

The proof is given in Appendix~\ref{app:full_boundary_bridge}. 
Proposition~\ref{prop:full_boundary_bridge} implies that, once the complete boundary state is specified, VTV-FM obtains both the second-order trajectory and the endpoint cost from a single variational principle. 
In the observed-\(v_1\) regime, this full-boundary bridge can be used directly for pairing and acceleration supervision. 
In the latent-\(v_1\) regime, the same cost \(C_T\) provides the starting point for eliminating the unobserved terminal velocity.

\subsection{Terminal-Velocity Closure as Variational Projection}
\label{sec:terminal_velocity_projection}
% \vspace{-0.25em}
In settings where only the observable boundary tuple \((x_0,v_0,x_1)\) is observed,
the full-boundary bridge in Proposition~\ref{prop:full_boundary_bridge} cannot be
instantiated directly, since it depends on the terminal velocity \(v_1\).
We determine this unobserved terminal velocity by projecting the full-boundary cost onto the observable state space. Rather than relying on a heuristic, we minimize \(C_T(x_0,v_0,x_1,v_1)\) with respect to \(v_1\). The following theorem shows that minimizing \(C_T\) over the latent terminal
velocity yields a closed-form terminal-velocity closure \(v_1^\star\), and the
associated projected cost \(C_T^\star\).

% While this full-boundary bridge can be used directly for temporal data where the terminal velocity \(v_1\) is observed, static datasets only provide the terminal position \(x_1\). To address the missing terminal velocity in static data, we treat \(v_1\) as a latent variable. Rather than prescribing it via a hand-designed heuristic, we determine \(v_1\) by projecting the full-boundary variational cost onto the observable state space.

\begin{theorem}[Variational terminal-velocity projection]
\label{thm:terminal_velocity_projection}
For any \(T>0\) and observable boundary tuple \((x_0,v_0,x_1)\), the minimizer of the full-boundary cost over the latent terminal velocity is
\begin{equation}\label{eqn_lat_terminal_v}
    v_1^\star(x_0,v_0,x_1) =
\arg\min_{v_1\in\mathcal V}
C_T(x_0,v_0,x_1,v_1)
= \frac{3}{2T}(x_1-x_0)-\frac{1}{2}v_0,
\end{equation}
and the corresponding projected transport cost is
\begin{equation}\label{eqn_proj_cost}
    C_T^\star(x_0,v_0,x_1)
=
\inf_{v_1\in\mathcal V}
C_T(x_0,v_0,x_1,v_1)
=
\frac{3}{T}
\left\|
\frac{x_1-x_0}{T}-v_0
\right\|_2^2 .
\end{equation}

Moreover, this projection is exact: minimizing the projected cost over observable couplings is equivalent to minimizing the full-boundary cost over all lifted couplings with latent terminal velocities.
\end{theorem}

The proof is given in Appendix~\ref{app:terminal_velocity_projection}. Theorem~\ref{thm:terminal_velocity_projection} provides the variational projection of the full-boundary cost onto the observable boundary variables. In this construction, the missing terminal velocity is not supplied by an auxiliary model or an external heuristic; instead, it is the analytic optimizer of the same minimum-acceleration principle that defines the bridge and the endpoint cost. In the latent-\(v_1\) regime, VTV-FM therefore
uses the completed boundary state
\((x_0,v_0,x_1,v_1^\star(x_0,v_0,x_1))\) to instantiate the full-boundary bridge
in Proposition~\ref{prop:full_boundary_bridge}, yielding the trajectory
\(x_t^\star\), velocity \(v_t^\star\), and acceleration target \(a_t^\star\).
% Theorem~\ref{thm:terminal_velocity_projection} shows that the missing terminal velocity is not supplied by an auxiliary model or an external heuristic. 
% Instead, it is the analytic optimizer of the same minimum-acceleration principle that defines the bridge and the endpoint cost. 

Consequently, both regimes use the same full-boundary bridge construction. When \(v_1\) is observed, VTV-FM uses the full-boundary cost \(C_T\) for pairing and the observed \(v_1\) for bridge instantiation. When \(v_1\) is latent, it
uses the projected cost \(C_T^\star\) for pairing and the variational closure \(v_1^\star\) for bridge instantiation. The next section describes how \(C_T\) and \(C_T^\star\) are used to construct the minibatch OT cost matrix and generate trajectory-level supervision.

%The resulting costs are instantiated on minibatches in Section~\ref{sec:minibatch_pairing}.

\subsection{Mini-batch Pairing and Trajectory Supervision}
\label{sec:minibatch_pairing}
% \vspace{-0.25em}
% Theorem~\ref{thm:terminal_velocity_projection} guarantees that the projected cost \(C_T^\star\) in Eq. (\ref{eqn_proj_cost}) yields the exact phase-space optimal transport coupling even when terminal velocities are missing. 

Theorem~\ref{thm:terminal_velocity_projection} justifies using \(C_T^\star\) in
Eq.~\eqref{eqn_proj_cost} for latent-\(v_1\) pairing: it gives the observable OT
problem equivalent to the full-boundary formulation after optimizing out \(v_1\). We now compute the variational costs on mini-batches to compute source--target pairings and trajectory-level supervision. 

Given a source mini-batch $\mathcal B_0=\{(x_0^{(i)},v_0^{(i)})\}_{i=1}^{m}\sim \mu_0$ and a target mini-batch $\mathcal B_1=\{x_1^{(j)}\}_{j=1}^{n}\sim \rho_1$, 
we can construct a pairing cost matrix \(C\in\mathbb R^{m\times n}\) by
\begin{equation}
C_{ij}
=
\begin{cases}
C_T(x_0^{(i)},v_0^{(i)},x_1^{(j)},v_1^{(j)}),
&
\text{observed } v_1,\\[4pt]
C_T^\star(x_0^{(i)},v_0^{(i)},x_1^{(j)}),
&
\text{latent } v_1.
\end{cases}
\label{eq:regime_cost_matrix}
\end{equation}
Here \(v_1^{(j)}\) denotes the terminal velocity estimated from temporal differences, simulator derivatives, or physical derivative information when available. We then solve the entropic OT problem with Sinkhorn iterations and use the resulting minibatch OT plan to select matched pairs \((i,j)\). Details of the OT
solver are provided in Appendix~\ref{app:minibatch_pairing}.

% We then solve the entropic OT problem with Sinkhorn iterations and get matched pairs. All details of solving OT in VTV-FM are in Appendix \ref{app:minibatch_pairing}.

To compute the bridge trajectory and acceleration target for each selected pair \((i,j)\), we first complete its terminal velocity:
\begin{equation}\label{eq:completed_terminal_velocity}
    \widetilde v_1^{(i,j)} =
\begin{cases}
v_1^{(j)},
&
\text{observed } v_1,\\[4pt]
v_1^\star(x_0^{(i)},v_0^{(i)},x_1^{(j)}),
&
\text{latent } v_1.
\end{cases}
\end{equation}
Thus, both regimes produce a complete boundary tuple $\zeta^{(i,j)} = \bigl(x_0^{(i)},v_0^{(i)},x_1^{(j)},\widetilde v_1^{(i,j)}\bigr)$. For \(t\sim\mathrm{Unif}(0,T)\), we instantiate the full-boundary bridge in Proposition~\ref{prop:full_boundary_bridge} as
\[
x_t^\star=\gamma_T^\star(t;\zeta^{(i,j)}),\qquad
v_t^\star=\dot\gamma_T^\star(t;\zeta^{(i,j)}),\qquad
a_t^\star=\ddot\gamma_T^\star(t;\zeta^{(i,j)}).
\]

% and then instantiate the variational bridge in Proposition~\ref{prop:full_boundary_bridge}. 
% For \(t\sim\mathrm{Unif}(0,T)\), this gives
% \[
% x_t^\star
% =
% \gamma_T^\star
% \bigl(t;x_0^{(i)},v_0^{(i)},x_1^{(j)},\widetilde v_1^{(i,j)}\bigr),
% \]
% \[
% v_t^\star
% =
% \dot\gamma_T^\star
% \bigl(t;x_0^{(i)},v_0^{(i)},x_1^{(j)},\widetilde v_1^{(i,j)}\bigr),
% \]
% \[
% a_t^\star
% =
% \ddot\gamma_T^\star
% \bigl(t;x_0^{(i)},v_0^{(i)},x_1^{(j)},\widetilde v_1^{(i,j)}\bigr).
% \]

The tuple \((x_t^\star,v_t^\star,t/T,a_t^\star)\) provides the trajectory supervision to train the acceleration field. This construction unifies pairing and supervision: both the matching cost and the acceleration target are induced by the same minimum-acceleration bridge, with the latent-\(v_1\) closure given by Theorem~\ref{thm:terminal_velocity_projection}.

% The tuple \((x_t^\star,v_t^\star,t/T,a_t^\star)\) provides the trajectory supervision used to train the acceleration field. %in Section~\ref{sec:acceleration_matching}. 
% This construction couples pairing and supervision through the same variational principle. The cost to match samples and the acceleration target used for training are both induced by the full-boundary minimum-acceleration bridge, with the latent-\(v_1\) case obtained by the variational projection in Theorem~\ref{thm:terminal_velocity_projection}.

\subsection{Acceleration Matching and Sampling}
\label{sec:acceleration_matching}
% \vspace{-0.25em}
Given the trajectory supervision \((x_t^\star,v_t^\star,\tau,a_t^\star)\), with \(\tau=t/T\), we train a neural acceleration field \(a_\theta:\mathcal X\times\mathcal V\times[0,1]\to\mathcal V\). 
The primary objective is \emph{acceleration matching}, which regresses the acceleration target induced by the variational bridge:
\begin{equation}\label{eqn_loss_acc}
    \mathcal L_{\mathrm{acc}}(\theta) =
\mathbb E
\left[
\left\|
a_\theta(x_t^\star,v_t^\star,\tau)
-
a_t^\star
\right\|_2^2
\right],
\end{equation}
where the expectation is over source samples, mini-batch OT pairings, and sampled times.

% We optionally add two lightweight regularizers. 
% First, \emph{local phase-space consistency} encourages the learned acceleration to reproduce the short-time evolution of the same variational bridge. 
% The velocity-consistency term enforces \(v_{t+\Delta t}\approx v_t+\Delta t\,a_\theta\):
% \[
% \mathcal L_{\mathrm{va}}
% =
% \mathbb E
% \left[
% \left\|
% v_{t+\Delta t}^\star
% -
% v_t^\star
% -
% \Delta t\,a_\theta(x_t^\star,v_t^\star,\tau)
% \right\|_2^2
% \right],
% \]
% and the position-consistency term enforces the second-order update \(x_{t+\Delta t}\approx x_t+\Delta t\,v_t+\frac12\Delta t^2a_\theta\):
% \[
% \mathcal L_{\mathrm{xv}}
% =
% \mathbb E
% \left[
% \left\|
% x_{t+\Delta t}^\star
% -
% x_t^\star
% -
% \Delta t\,v_t^\star
% -
% \frac{1}{2}\Delta t^2 a_\theta(x_t^\star,v_t^\star,\tau)
% \right\|_2^2
% \right].
% \]
% Second, inspired by flow-divergence matching~\citep{huang2025aligningdivergence}, we add \emph{acceleration-divergence matching} (ADM). 
% For the phase-space vector field \(b_\theta(x,v,t)=(v,a_\theta(x,v,t))\), its divergence is \(\nabla_z\cdot b_\theta=\nabla_v\cdot a_\theta\). 
% Using the divergence target induced by the minimum-acceleration bridge, we regularize
% \[
% \mathcal L_{\mathrm{ADM}}
% =
% \mathbb E
% \left[
% \left(
% \frac{T-t}{d}
% \nabla_v\cdot a_\theta(x_t^\star,v_t^\star,\tau)
% +
% 3
% \right)^2
% \right].
% \]
% This term encourages the learned phase-space dynamics to match the local contraction or expansion structure of the variational bridge.

We optionally add local phase-space consistency to encourage the learned dynamics to reproduce the short-time evolution of the same variational bridge. For a small step \(\Delta t\), we use
\begin{subequations}\label{eqn_loss_local_consistency}
\begin{align}
\mathcal L_{\mathrm{va}}
&=
\mathbb E
\left[
\left\|
v_{t+\Delta t}^\star
-
v_t^\star
-
\Delta t\,a_\theta(x_t^\star,v_t^\star,\tau)
\right\|_2^2
\right],
\label{eqn_loss_va}
\\
\mathcal L_{\mathrm{xv}}
&=
\mathbb E
\left[
\left\|
x_{t+\Delta t}^\star
-
x_t^\star
-
\Delta t\,v_t^\star
-
\frac{1}{2}\Delta t^2 a_\theta(x_t^\star,v_t^\star,\tau)
\right\|_2^2
\right].
\label{eqn_loss_xv}
\end{align}
\end{subequations}
Here \(\mathcal L_{\mathrm{va}}\) enforces local velocity consistency, while \(\mathcal L_{\mathrm{xv}}\) enforces the corresponding second-order position consistency.

Besides, inspired by flow-divergence matching~\citep{huang2025flowdivergence}, we introduce \emph{acceleration-divergence matching} (ADM), a second-order phase-space analogue that regularizes the divergence of the learned acceleration dynamics. We match the bridge-induced divergence target through
\begin{equation}
\label{eqn_loss_adm}
\mathcal L_{\mathrm{ADM}}
=
\mathbb E\!\left[
\left(
\frac{T-t}{d}\nabla_v\!\cdot a_\theta(x_t^\star,v_t^\star,\tau)+3
\right)^2
\right].
\end{equation}
Here, \(+3\) corresponds to the latent-\(v_1\) regime; for fixed observed
\(v_1\), it becomes \(+4\) (Appendix~\ref{app:acceleration_divergence_matching}).
This term encourages the learned phase-space dynamics to match the local contraction or expansion structure of the variational bridge. Detailed derivations and implementation details are provided in Appendices~\ref{app:local_integral_consistency} and~\ref{app:acceleration_divergence_matching}. The total objective is
\[
\mathcal L_{\mathrm{total}}
=
\mathcal L_{\mathrm{acc}}
+
\lambda_{\mathrm{va}}\mathcal L_{\mathrm{va}}
+
\lambda_{\mathrm{xv}}\mathcal L_{\mathrm{xv}}
+
\lambda_{\mathrm{ADM}}\mathcal L_{\mathrm{ADM}}.
\]

\textbf{Sampling.} After training, samples are generated by integrating the learned second-order dynamics
\[
\dot x_t=v_t,
\qquad
\dot v_t=a_\theta(x_t,v_t,t/T)
\]
from an initial phase-space sample \((x_0,v_0)\sim\mu_0\). 
The generated sample is the terminal position \(x_T\). 
In practice, rather than first-order Euler discretization, we use fixed-step solvers better aligned with second-order phase-space dynamics, including Heun, Taylor2, and Velocity Verlet~\citep{hairer1993solving,butcher2016numerical,swope1982computer}. For Velocity Verlet, with $\Delta t=T/N$ and
$a_k=a_\theta(x_k,v_k,t_k/T)$, we first update the position as
$x_{k+1}=x_k+\Delta t\,v_k+\frac{1}{2}\Delta t^2 a_k$.
We then evaluate
$a_{k+1}=a_\theta(x_{k+1},v_k+\Delta t\,a_k,t_{k+1}/T)$
and correct the velocity by
$v_{k+1}=v_k+\frac{\Delta t}{2}(a_k+a_{k+1})$.
Complete update rules for Heun, Taylor2, and Velocity Verlet are
provided in Appendix~\ref{app:numerical_solvers}.
Algorithm~\ref{alg:vtvfm} summarizes the training and sampling procedure:
\begin{algorithm}[h]
\caption{\textsc{VTV-FM}: Variational Terminal-Velocity Flow Matching}
\label{alg:vtvfm}
\begin{algorithmic}[1]
\Require Source batch $B_0=\{(x_0^i,v_0^i)\}_{i=1}^m$, and target batch $B_1=\{x_1^j\}_{j=1}^n$, horizon $T$, acceleration model $a_\theta$

\For{each training iteration}

    \State For the latent-\(v_1\) regime, compute \(v_1^\star\) by Eq.~(\ref{eqn_lat_terminal_v}) and obtain \(C_T^\star\) by Eq.~(\ref{eqn_proj_cost}).

    \State Construct the mini-batch OT cost matrix \(C_{ij}\) by Eq.~(\ref{eq:regime_cost_matrix}).

    \State Solve mini-batch OT by Sinkhorn and obtain matched pairs \((i,j)\).

    \State Complete the terminal velocity \(\widetilde v_1^{(i,j)}\) by Eq.~(\ref{eq:completed_terminal_velocity}).

    \State Sample \(t\sim \mathrm{Unif}(0,T)\) and set \(\tau=t/T\).

    \State Compute trajectory supervision \((x_t^\star,v_t^\star,a_t^\star)\) using the full-boundary bridge in Eqs.~(\ref{eqn_hermite_bridge_x}),~(\ref{eqn_hermite_bridge_v}).
\label{}
    \State Update \(a_\theta\) by minimizing Eq.~(\ref{eqn_loss_acc}), optionally with the regularizers in Eqs.~(\ref{eqn_loss_va}),~(\ref{eqn_loss_xv}), ~(\ref{eqn_loss_adm}).

\EndFor

\State \textbf{Sampling:} draw \((x_0,v_0)\sim\mu_0\), integrate the learned second-order dynamics, and return \(x_T\).

\end{algorithmic}
\end{algorithm}
\vspace{-0.25em}

\vspace{-0.50em}
\section{Experiments}
\label{sec:experiments}
\vspace{-0.25em}

We quantitatively evaluate VTV-FM on low-dimensional datasets,
PDE-governed physical fields ~\citep{lippe2023pderefiner}, and
CIFAR-10~\citep{krizhevsky2009learning}, and additionally provide
qualitative proof-of-concept temporal examples on BAIR robot
pushing~\citep{ebert2017selfsupervised} and
KTH~\citep{schuldt2004recognizing}.
The low-dimensional datasets and CIFAR-10 evaluate the missing-\(v_1\)
regime, where terminal velocity is unobserved and must be inferred by
the proposed closure.
The low-dimensional setting further enables direct visualization of
sample quality and transport trajectories.
In contrast, the PDE and temporal examples evaluate the
observed/derivable-\(v_1\) regime, where terminal derivative information
is available from governing dynamics, simulators, or adjacent frames.
Additional qualitative temporal visualizations are provided in
Appendix~\ref{app:video_vis}.

%\vspace{-0.25em}
\subsection{Low-Dimensional Case Studies}
\label{sec:toy}
%\vspace{-0.25em}

We evaluate VTV-FM on low-dimensional datasets, where generated samples and transport trajectories can be directly visualized.
This setting compares sample quality and transport geometry on complex two-dimensional structures.
Loss and solver analyses are provided in Appendix~\ref{app:toy_ablation_solver}.

\textbf{Datasets.}
We use seven two-dimensional datasets: Dinosaur, Swiss Roll, Checkerboard, Two Moons, Parallel Lines, Tree, and Star.
These datasets cover curved manifolds, multimodal structures, branching patterns, and sharp local direction changes, making them suitable for testing how different flow constructions model low-dimensional geometry.

\textbf{Baselines.}
We compare VTV-FM with CFM~\citep{lipman2023flowmatching}, OT-CFM~\citep{tong2024minibatchot}, OAT-FM~\citep{yue2025oatfm}, HOFM~\citep{su2025highorderflowmatching}, and CurveFlow~\citep{luo2025curveflow}.
CFM and OT-CFM are first-order baselines, where OT-CFM improves source--target pairing through minibatch OT while retaining a first-order linear conditional bridge.
OAT-FM, HOFM, and CurveFlow are stronger acceleration-based or curve-based baselines, allowing us to compare whether high-order path modeling improves low-dimensional generation quality.

\textbf{Metric.}
We evaluate generation quality using sample-based \(W_2\).
For each dataset, we denote the generated samples by \(\{\hat{x}_i\}_{i=1}^n\) and the target samples by \(\{y_i\}_{i=1}^n\).
The corresponding empirical distributions are \(\widehat{\rho}_{\mathrm{gen}}=\sum_i\delta_{\hat{x}_i}/n\) and \(\widehat{\rho}_{\mathrm{data}}=\sum_i\delta_{y_i}/n\).
We report \(W_2(\widehat{\rho}_{\mathrm{gen}},\widehat{\rho}_{\mathrm{data}})=\left(\min_{\pi\in\Pi}\sum_{i,j}\pi_{ij}\|\hat{x}_i-y_j\|_2^2\right)^{1/2}\), where \(\Pi\) denotes the set of couplings between the two uniform empirical distributions.
Lower values indicate better agreement between generated and target samples.
For the radar plot, we convert the raw \(W_2\) values into a per-dataset relative score.
Given dataset \(d\) and method \(m\), we define \(s_{d,m}=\min_{m'}W_2(d,m')/W_2(d,m)\).
Thus, the best method on each dataset has score \(1\), and higher radar scores indicate better relative performance.

\textbf{Results and findings.}
Table~\ref{tab:toy_main} reports the main results on seven low-dimensional datasets.
VTV-FM achieves the best average sample-based \(W_2\), reaching \(0.1004\) and outperforming all baselines.
Compared with CFM, OT-CFM, OAT-FM, HOFM, and CurveFlow, VTV-FM reduces the average \(W_2\) by approximately \(9.1\%\), \(6.8\%\), \(4.1\%\), \(3.4\%\), and \(2.4\%\), respectively.
The radar relative score shows the same trend: VTV-FM obtains the highest average score of \(0.987\), compared with \(0.963\) for CurveFlow and \(0.955\) for HOFM, indicating more stable relative performance across datasets.

At the dataset level, VTV-FM achieves the best result on \(4/7\) datasets: Checkerboard, Two Moons, Parallel Lines, and Star.
Compared with the strongest competing baseline on these datasets, VTV-FM reduces \(W_2\) by approximately \(3.0\%\), \(6.2\%\), \(0.6\%\), and \(1.7\%\), respectively.
Checkerboard, Two Moons, and Star involve multimodality, separated modes, or sharp local direction changes, suggesting that second-order phase-space dynamics are particularly beneficial for complex low-dimensional structures.
On another three datasets, VTV-FM is not the best method but remains competitive.
On Tree, it is only about \(0.8\%\) worse than the best HOFM result.
This suggests that high-order or OT-based path modeling can still be effective for certain geometries.
However, in terms of average performance, VTV-FM shows the most robust behavior across low-dimensional structures.
Additional controlled analyses of the source phase-space prior,
terminal-velocity completion, pairing--bridge consistency, and
computational efficiency on the low-dimensional benchmarks are provided
in Appendices~\ref{app:source_prior_ablation}--\ref{app:efficiency_scaling}.

\begin{figure*}[t]
\centering

\begin{minipage}[t]{0.64\textwidth}
\vspace{0pt}
\captionof{table}{
Low-dimensional results measured by sample-based \(W_2\).
Lower is better.
The right radar plot visualizes the same results using the relative score
\(s_{d,m}\), where higher is better.
}
\label{tab:toy_main}
%\vspace{-0.6em}

\centering
\small
\setlength{\tabcolsep}{2.6pt}
\renewcommand{\arraystretch}{1.05}
\resizebox{\linewidth}{!}{
\begin{tabular}{lcccccc}
\toprule
Dataset & CFM & OT-CFM & OAT-FM & HOFM & CurveFlow & VTV-FM \\
\midrule
Dinosaur        & 0.1336 & 0.1302 & 0.1237 & \textbf{0.1064} & 0.1088 & 0.1109 \\
Swiss Roll      & 0.1083 & \textbf{0.1037} & 0.1064 & 0.1209 & 0.1176 & 0.1081 \\
Checkerboard    & 0.1229 & 0.1207 & 0.1148 & 0.1137 & 0.1126 & \textbf{0.1092} \\
Two Moons       & 0.0978 & 0.0956 & 0.0884 & 0.0857 & 0.0842 & \textbf{0.0790} \\
Parallel Lines  & 0.1441 & 0.1381 & 0.1374 & 0.1391 & 0.1370 & \textbf{0.1362} \\
Tree            & 0.0760 & 0.0744 & 0.0737 & \textbf{0.0724} & 0.0727 & 0.0730 \\
Star            & 0.0902 & 0.0915 & 0.0885 & 0.0893 & 0.0877 & \textbf{0.0862} \\
\midrule
Avg.            & 0.1104 & 0.1077 & 0.1047 & 0.1039 & 0.1029 & \textbf{0.1004} \\
\bottomrule
\end{tabular}
}
\end{minipage}
\hfill
\begin{minipage}[t]{0.35\textwidth}
%\vspace{0pt}
\vspace*{0.8em}
\centering
\includegraphics[width=\linewidth]{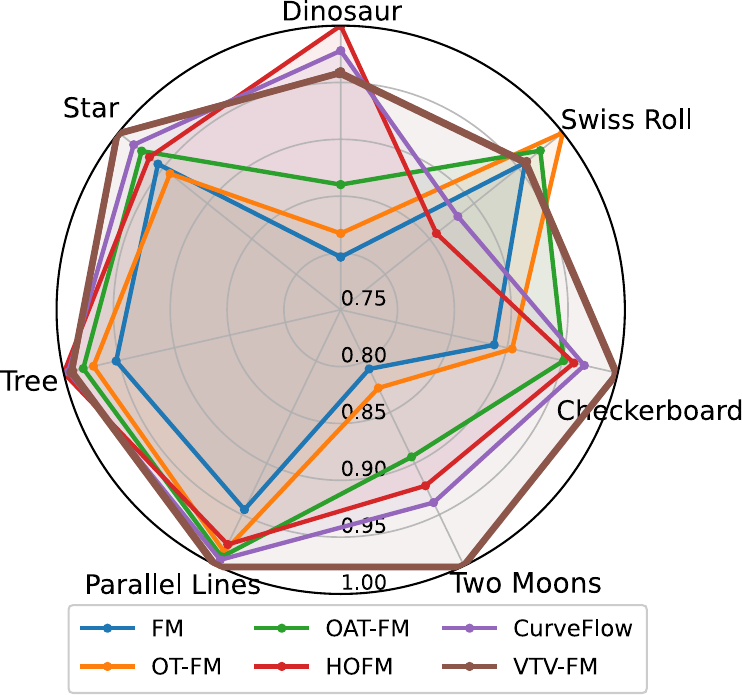}
%\vspace{-0.5em}

\end{minipage}

\vspace{-0.8em}
\end{figure*}

% \vspace{-0.25em}
\subsection{PDE Case Studies in the Observed-\texorpdfstring{$v_1$}{v1} Regime}
\label{sec:pde}

We evaluate VTV-FM on PDE-governed generation tasks.
Unlike two-dimensional datasets and CIFAR-10, PDE data provide explicit temporal evolution, so the terminal derivative can be obtained from the governing equation or numerical simulator.
This corresponds to the observed-\(v_1\) regime, where VTV-FM can use the full-boundary second-order bridge without terminal-velocity closure.

\textbf{Datasets.}
We consider two two-dimensional PDE datasets: Navier--Stokes vorticity generation (NS2D) and Kuramoto--Sivashinsky generation (KS2D).
For NS2D, the state variable is the vorticity field \(\omega\), governed by \(\partial_t\omega+u\cdot\nabla\omega=\nu\Delta\omega+f\).
For KS2D, the state variable is a scalar field \(u\), governed by \(\partial_tu+\frac{1}{2}|\nabla u|^2+\Delta u+\Delta^2u=0\).
In both datasets, the terminal derivative can be obtained from the PDE right-hand side or simulator outputs.

% \textbf{Baselines.}
% We compare VTV-FM with OT-CFM and HOFM.
% OT-CFM is the stronger first-order OT-based baseline in Section~\ref{sec:toy}, while HOFM is a competitive high-order baseline.
% This comparison tests whether VTV-FM can further benefit from the observed terminal derivative and the full-boundary second-order bridge.

\textbf{Baselines.}
We compare VTV-FM with OT-CFM and HOFM.
OT-CFM is the stronger first-order OT-based baseline, while HOFM is a competitive high-order baseline.
This comparison evaluates whether terminal derivative information and the full-boundary bridge improve second-order generation.

\textbf{Metrics.}
We evaluate PDE generation quality using MMD, PSD log-\(L_2\), and a derivative-sensitive gap.
MMD~\citep{gretton2012kernel} measures global distributional discrepancy, PSD log-\(L_2\)
measures multiscale spectral fidelity, and the derivative-sensitive gap
compares the distributions of per-sample \(L_2\) norms of PDE-related quantities.
For NS2D, we use \(\nu\Delta\omega\); for KS2D, we use the PDE right-hand side.

\begin{wraptable}{r}{0.5\linewidth}
\vspace{-1.0em}
\centering
\caption{
PDE generation results.
}
\label{tab:pde_results}
%\scriptsize
\footnotesize
\setlength{\tabcolsep}{3.0pt}
\renewcommand{\arraystretch}{0.88}
\begin{tabular}{llccc}
\toprule
Dataset & Method & MMD & PSD log-\(L_2\) & Deriv. gap \\
\midrule
\multirow{3}{*}{NS2D}
& OT-CFM  & 0.017 & 5.08 & \(6.93{\times}10^{-3}\) \\
& HOFM    & 0.015 & 4.22 & \(1.06{\times}10^{-3}\) \\
& VTV-FM  & \textbf{0.014} & \textbf{3.48} & \(\mathbf{2.80{\times}10^{-4}}\) \\
\midrule
\multirow{3}{*}{KS2D}
& OT-CFM  & 0.035 & 9.00 & \(5.50{\times}10^{-1}\) \\
& HOFM    & 0.015 & 8.13 & \(9.17{\times}10^{-2}\) \\
& VTV-FM  & \textbf{0.010} & \textbf{7.52} & \(\mathbf{1.30{\times}10^{-2}}\) \\
\bottomrule
\end{tabular}
\vspace{-1.0em}
\end{wraptable}

\textbf{Results and findings.}
Table~\ref{tab:pde_results} reports the PDE generation results on NS2D and KS2D.
VTV-FM achieves the best result on both PDE datasets across all metrics, improving global distributional matching, spectral fidelity, and PDE-related derivative structure.

On NS2D, VTV-FM reduces MMD from \(0.017\) (OT-CFM) and \(0.015\) (HOFM) to \(0.014\), corresponding to relative reductions of approximately \(17.6\%\) and \(6.7\%\), respectively.
For spectral structure, it reduces PSD log-\(L_2\) from \(5.08\) and \(4.22\) to \(3.48\), giving reductions of approximately \(31.5\%\) over OT-CFM and \(17.5\%\) over HOFM.
The largest gain appears in the derivative-sensitive gap, VTV-FM reduces it from \(6.93{\times}10^{-3}\) and \(1.06{\times}10^{-3}\) to \(2.80{\times}10^{-4}\), corresponding to reductions of approximately \(96.0\%\) and \(73.6\%\).
This indicates that VTV-FM better preserves diffusion-related derivative structure.

On KS2D, VTV-FM also consistently improves over both baselines.
It reduces MMD from \(0.035\) (OT-CFM) and \(0.015\) (HOFM) to \(0.010\), yielding relative reductions of approximately \(71.4\%\) and \(33.3\%\), respectively.
For PSD log-\(L_2\), VTV-FM reduces the value from \(9.00\) and \(8.13\) to \(7.52\), corresponding to reductions of approximately \(16.4\%\) and \(7.5\%\).
For the derivative-sensitive gap, it reduces the value from \(5.50{\times}10^{-1}\) and \(9.17{\times}10^{-2}\) to \(1.30{\times}10^{-2}\), giving reductions of approximately \(97.6\%\) and \(85.8\%\).
This suggests that VTV-FM improves not only overall distribution matching, but also the PDE-related dynamical structure of the KS system.

Overall, the most pronounced improvements appear in the derivative-sensitive gap on both PDE systems.
These results show that VTV-FM can directly exploit terminal derivative information and unify it with the full-boundary bridge geometry and trajectory supervision.

% \begin{table}[t]
% \centering
% \caption{
% PDE generation results in the observed-\(v_1\) regime.
% Lower is better for all metrics.
% }
% \label{tab:pde_results}
% \small
% \setlength{\tabcolsep}{5.5pt}
% \begin{tabular}{llccc}
% \toprule
% Dataset & Method & MMD & PSD log-\(L_2\) & Derivative gap \\
% \midrule
% \multirow{3}{*}{NS2D}
% & OT-CFM  & 0.017 & 5.08 & \(6.93{\times}10^{-3}\) \\
% & HOFM    & 0.015 & 4.22 & \(1.06{\times}10^{-3}\) \\
% & VTV-FM  & \textbf{0.014} & \textbf{3.48} & \(\mathbf{2.80{\times}10^{-4}}\) \\
% \midrule
% \multirow{3}{*}{KS2D}
% & OT-CFM  & 0.035 & 9.00 & \(5.50{\times}10^{-1}\) \\
% & HOFM    & 0.015 & 8.13 & \(9.17{\times}10^{-2}\) \\
% & VTV-FM  & \textbf{0.010} & \textbf{7.52} & \(\mathbf{1.30{\times}10^{-2}}\) \\
% \bottomrule
% \end{tabular}
% \vspace{-0.8em}
% \end{table} 

\vspace{-0.25em}

\subsection{CIFAR-10 Case Study in the Missing-\texorpdfstring{$v_1$}{v1} Regime}
\label{sec:cifar10}
\vspace{-0.25em}

We next evaluate VTV-FM on CIFAR-10, a standard image-generation benchmark where target-side terminal derivatives are unavailable.
This corresponds to the missing-\(v_1\) regime: only target samples are observed, and the second-order bridge must be instantiated using the terminal-velocity closure derived in Section~\ref{sec:method}.
This experiment tests whether the proposed construction remains effective in high-dimensional image generation.

\textbf{Baselines and evaluation.}
We compare VTV-FM with representative Flow Matching baselines, including FM, I-CFM, OT-CFM, OAT-FM, HOFM, and CurveFlow.
FM and I-CFM use independent source--target coupling, while OT-CFM uses minibatch OT pairing.
OAT-FM further refines first-order couplings through acceleration transport, whereas HOFM and CurveFlow provide high-order or curve-based path-modeling baselines.
Because these methods use different dynamical orders and native sampling procedures, we report both solver-controlled comparisons and method-specific inference protocols.
For the controlled cross-method comparison, VTV-FM uses acceleration matching only, without the local consistency or ADM regularizers.

\begin{table}[!htbp]
\centering
\caption{
CIFAR-10 FID under controlled inference protocols.
For adaptive Dopri5, the measured number of function evaluations (NFE) is shown in parentheses.
}
\label{tab:cifar10_solver_controlled}
\small
\setlength{\tabcolsep}{4.5pt}
\renewcommand{\arraystretch}{1.0}
\begin{tabular}{lccccc}
\toprule
Method
& Euler-100
& Dopri5
& Heun-100
& Taylor2-100
& Verlet-100 \\
\midrule
FM
& 4.58 & 3.73 (147) & 3.62 & -- & -- \\
I-CFM
& 4.48 & 3.68 (146) & 3.59 & -- & -- \\
OT-CFM
& 4.46 & 3.60 (134) & 3.48 & -- & -- \\
FM + OAT-FM
& 3.91 & 3.56 (135) & 3.51 & -- & -- \\
I-CFM + OAT-FM
& 3.77 & 3.50 (138) & 3.45 & -- & -- \\
OT-CFM + OAT-FM
& 3.72 & 3.47 (126) & 3.42 & -- & -- \\
CurveFlow
& 3.47 & 3.41 (130) & 3.40 & -- & -- \\
HOFM
& -- & -- & -- & 3.45 & -- \\
\midrule
VTV-FM (Acc. only)
& \textbf{3.28}
& \textbf{3.23 (81)}
& \textbf{3.17}
& \textbf{3.25}
& \textbf{3.10} \\
\bottomrule
\end{tabular}
\vspace{-0.5em}
\end{table}

\textbf{Solver-controlled comparison.}
Table~\ref{tab:cifar10_solver_controlled} separates method design from solver choice.
Under the shared Heun-100 protocol with 200 NFE, FM, OT-CFM, and VTV-FM obtain FIDs of \(3.62\), \(3.48\), and \(3.17\), respectively.
Under the same adaptive Dopri5 tolerance, VTV-FM reaches an FID of \(3.23\) with \(81\) NFE, compared with \(3.73\) with \(147\) NFE for FM and \(3.60\) with \(134\) NFE for OT-CFM.
These results show that the improvement persists under shared general-purpose solvers and therefore cannot be attributed solely to Velocity Verlet, auxiliary regularization, or a larger sampling budget.

\textbf{Role of the auxiliary objectives.}
We next evaluate the effect of local consistency and acceleration-divergence matching (ADM) within VTV-FM.
Table~\ref{tab:cifar10_ablation} reports the results under three fixed-step phase-space solvers.
Acceleration matching alone already achieves FIDs of \(3.10\), \(3.17\), and \(3.25\) with Velocity Verlet, Heun, and Taylor2, respectively.
Adding local consistency improves these values to \(3.06\), \(3.16\), and \(3.23\).
Adding ADM alone gives \(3.05\), \(3.14\), and \(3.19\), showing a larger individual improvement than local consistency under all three solvers.
Combining both auxiliary terms gives the best overall results, including an FID of \(2.99\) with 100-step Velocity Verlet sampling.

\begin{wraptable}[9]{r}{0.46\linewidth}
\vspace{-0.7em}
\centering
\caption{
VTV-FM objective ablation.
}
\label{tab:cifar10_ablation}
\footnotesize
\setlength{\tabcolsep}{2.8pt}
\renewcommand{\arraystretch}{0.92}
\resizebox{\linewidth}{!}{
\begin{tabular}{@{}lccc@{}}
\toprule
Objective & Verlet & Heun & Taylor2 \\
\midrule
Acc.
& 3.10 & 3.17 & 3.25 \\
Acc. + Cons.
& 3.06 & 3.16 & 3.23 \\
Acc. + ADM
& 3.05 & 3.14 & 3.19 \\
Acc. + Cons. + ADM
& \textbf{2.99} & \textbf{3.13} & \textbf{3.18} \\
\bottomrule
\end{tabular}
}
\vspace{-0.7em}
\end{wraptable}

Here, \(\mathrm{Cons.}=\mathcal L_{\mathrm{va}}+\mathcal L_{\mathrm{xv}}\).
The ablation indicates that the main gain is already present with bridge-induced acceleration supervision, while local consistency and ADM provide complementary refinements.
ADM yields the larger individual gain, and combining both auxiliary terms gives the strongest overall performance.
Accordingly, we treat acceleration matching as the core learning objective and the consistency and ADM terms as optional refinements.

\textbf{Results and findings.}
The controlled comparisons show that VTV-FM improves over first-order FM and OT-based FM even when solver choice is held fixed.
In particular, under Heun-100, VTV-FM reduces FID from \(3.62\) for FM and \(3.48\) for OT-CFM to \(3.17\).
Under adaptive Dopri5, VTV-FM also achieves a lower FID with fewer function evaluations.
Within VTV-FM, the full objective further improves the best fixed-step result from \(3.10\) with acceleration matching alone to \(2.99\) when both consistency and ADM are included.

These results support two complementary conclusions.
First, the improvement of VTV-FM is not a solver-specific effect: it persists under shared Heun and Dopri5 inference protocols.
Second, the auxiliary regularizers refine an already strong acceleration-matching model rather than accounting for the primary gain.
Together, the results show that the variationally induced phase-space construction remains effective in the high-dimensional missing-\(v_1\) regime.
Qualitative CIFAR-10 samples are provided in
Appendix~\ref{app:cifar10_samples}.
Additional ImageNet-256 scaling results are reported in
Appendix~\ref{app:efficiency_scaling}.

\section{Related Work}
\vspace{-0.25em}
\textbf{Flow Matching and OT-based coupling.}
Flow Matching learns generative transport by regressing time-dependent
vector fields along prescribed probability paths, with stochastic
interpolants and rectified flows providing related formulations~
\citep{lipman2023flowmatching,albergo2023stochastic,liu2023flowstraightfast}.
Minibatch optimal transport improves source--target coupling and
simplifies transport paths and regression targets~
\citep{tong2024minibatchot,pooladian2023multisample}.
In OT-CFM with linear conditional bridges, optimal transport improves
sample pairing, but each conditional trajectory remains a straight
line between the paired endpoints.

\textbf{Higher-order dynamics and path design.}
Beyond improving sample coupling, existing methods enrich generative
dynamics through higher-order supervision and alternative dynamical
constructions.
HOMO introduces acceleration and higher-order derivative supervision
into one-step shortcut models~\citep{chen2025highordermatching},
and NRFlow jointly learns velocity and acceleration fields to improve
noise robustness and trajectory stability~\citep{chen2025nrflow}.
HOFM develops a unified framework incorporating arbitrary-order
trajectory derivatives, with conditional regression formulations,
derivative consistency constraints, and statistical convergence
rates~\citep{su2025highorderflowmatching}.
Related theory also studies estimation errors in higher-order
trajectory refinement~\citep{gong2025theoreticalhighorder}.
In terms of dynamical structure and path geometry, HRF models
velocity, acceleration, and higher-order variables through
hierarchically coupled rectified flows~\citep{zhang2025hrf},
and CurveFlow explores curvature-guided path objectives~
\citep{luo2025curveflow}.
For few-step generation, Terminal Velocity Matching matches the
derivative of a two-time displacement model with respect to its
terminal time~\citep{zhou2026terminal};
this derivative target is distinct from the data-side endpoint
velocity considered here.

\textbf{Minimum-acceleration transport and boundary conditions.}
A closely related variational approach uses squared-acceleration
action to construct trajectories and transport costs.
Classical spline theory shows that minimizing this action with
fixed endpoint positions and velocities yields a cubic Hermite
curve~\citep{deboor1978splines}.
At the distribution level, second-order optimal transport studies
acceleration-minimizing paths and cubic Wasserstein splines~
\citep{benamou2019secondorder}.
OAT-FM applies optimal acceleration transport to flow matching~
\citep{yue2025oatfm}, and OTIKIN studies second-order discrepancies
between distributions induced by acceleration costs~
\citep{brigati2025otikin}.
Endpoint information also plays an important role:
boundary-constrained flow matching imposes prescribed endpoint
behavior on learned first-order velocity fields~
\citep{hu2025boundaryconditions}.
Our setting concerns second-order phase-space bridges whose terminal
velocity is not provided by static data.
VTV-FM treats this unspecified velocity as an optimization variable
in the minimum-acceleration problem and eliminates it analytically,
jointly determining the terminal-velocity closure, bridge,
corresponding OT pairing cost, and acceleration supervision.

\vspace{-0.25em}
\section{Conclusion}
\vspace{-0.25em}
We introduced \textbf{VTV-FM}, a second-order FM framework for generative transport in phase space. 
The central challenge is that second-order bridges require a target-side terminal velocity, which is often unavailable for static data. 
VTV-FM addresses this issue with a minimum-acceleration variational principle: when the terminal velocity is missing, the principle yields an analytic terminal-velocity closure and a projected OT cost; when the terminal velocity is observed or derivable, the same full-boundary bridge directly incorporates it. This construction unifies terminal-velocity closure, source--target pairing, and acceleration supervision under a single variational framework. 
Experiments on low-dimensional datasets, PDE-governed fields, and CIFAR-10 show that VTV-FM improves generation quality over first-order, OT-based, and higher-order FM baselines, while producing more reliable second-order transport trajectories. 
Future work includes developing richer kinetic priors, extending the framework to conditional and controlled generation, and scaling second-order transport constructions to larger generative modeling tasks.

\vspace{-0.25em}
\section*{Acknowledgements and Disclosures}
\vspace{-0.25em}

\textbf{Funding.}
% The work done by H.\,Jiang and Y.\,He was supported in part by the National Science Foundation (NSF) under Grant Nos.\,IIS-2544353 (CAREER), IIS-2505719, IIS-2441449, and IOS-2446522, and in part by the Commonwealth Cyber Initiative (CCI). The work done by Y.\,Li and H. \,Chen was supported by the National Science Foundation under IIS-2348405, the Coastal Virginia Center for Cyber Innovation (COVA CCI), Commonwealth Cyber Initiative (CCI), the William \& Mary Faculty Research Grant, and generous gifts from NVIDIA and Amazon. And the work done by D.\,Yang and Y.\,Xiong was supported by the Coastal Virginia Center for Cyber Innovation (COVA CCI) and the Commonwealth Cyber Initiative (CCI).
The work done by H.\,Jiang and Y.\,He was supported in part by the National Science Foundation (NSF) under Grant Nos.\,IIS-2544353 (CAREER), IIS-2505719, IIS-2441449, and IOS-2446522, and in part by the Commonwealth Cyber Initiative (CCI). The work done by Y.\,Li and H. \,Chen was supported by the NSF under IIS-2348405, the Coastal Virginia Center for Cyber Innovation (COVA CCI), CCI, the William \& Mary Faculty Research Grant, and generous gifts from NVIDIA and Amazon. And the work done by D.\,Yang and Y.\,Xiong was supported by the COVA CCI and the CCI.

\textbf{Competing interests.}
The authors declare no financial relationships or competing interests outside the submitted work that could reasonably be perceived as influencing this work.

\bibliographystyle{unsrtnat}
% \bibliography{references}
\bibliography{references_camera_ready}

\clearpage
\section*{Technical Appendices and Supplementary Material}
\label{app:supplementary}

\startcontents[appendices]
\printcontents[appendices]{}{1}{
    \setcounter{tocdepth}{2}
    \vspace{-0.5em}
}

\appendix

% \section*{Technical appendices and supplementary material}
% \appendix

\newpage

\section{Limitations.}
VTV-FM relies on a minimum-acceleration principle to close the missing terminal velocity in static-data regimes. 
Although this gives an analytic and internally consistent bridge, richer kinetic structures or multimodal terminal velocities may require more flexible closures. 
Our experiments include low-dimensional datasets, PDE-governed fields,
CIFAR-10, an ImageNet-256 scaling study, and qualitative video examples.
Broader evaluation on higher-resolution image generation, quantitative
video generation, and conditional generation remains future work.

\vspace{-7pt}
\section{Broader impacts.}
\vspace{-7pt}
This work is primarily methodological and may benefit scientific generation, physical simulation, and data-driven modeling by improving transport geometry and derivative-aware generation. 
As with other generative modeling methods, improved sample quality may also be misused for synthetic image or video generation. 
The paper does not propose a deployed system or use private personal data, and our experiments are limited to standard benchmarks and simulated physical fields.
\vspace{-7pt}

\section{Core Derivations and Algorithm Details}
\label{app:core_derivations}

We prove the full-boundary bridge and terminal-velocity projection (subsections \ref{app:full_boundary_bridge} and \ref{app:terminal_velocity_projection}),
then provide minibatch pairing, auxiliary-loss, and sampling details (subsections \ref{app:minibatch_pairing}, \ref{app:local_integral_consistency}, \ref{app:acceleration_divergence_matching}, and \ref{app:numerical_solvers})
for Section~\ref{sec:method}.

\subsection{Proof of Proposition~\ref{prop:full_boundary_bridge}}
\label{app:full_boundary_bridge}

\begin{proof}
Let \(\Delta x=x_1-x_0\) and define
\[
\gamma^\star(t)=x_0+v_0t+\beta t^2+\alpha t^3,
\qquad
\alpha=\frac{-2\Delta x+T(v_0+v_1)}{T^3},
\qquad
\beta=\frac{3\Delta x-T(2v_0+v_1)}{T^2}.
\]
This polynomial satisfies all four boundary conditions and is
equivalent to the Hermite representation in
Eq.~\eqref{eqn_hermite_bridge_x}. Its derivatives give
Eq.~\eqref{eqn_hermite_bridge_v}.

For any other admissible \(\gamma\in\Gamma_T(x_0,v_0,x_1,v_1)\),
write \(r=\gamma-\gamma^\star\). Then
\(r(0)=r(T)=\dot r(0)=\dot r(T)=0\). Since
\((\gamma^\star)^{(4)}=0\), integration by parts gives
\[
\int_0^T\langle\ddot\gamma^\star,\ddot r\rangle\,dt=0.
\]
Consequently,
\[
\mathcal A_T[\gamma]
=\mathcal A_T[\gamma^\star]
+\int_0^T\|\ddot r(t)\|_2^2\,dt
\geq\mathcal A_T[\gamma^\star].
\]
Equality requires \(\ddot r=0\) almost everywhere; the homogeneous
boundary conditions then imply \(r=0\). Thus \(\gamma^\star\) is the
unique global minimizer in the admissible Sobolev class.

Finally, \(\ddot\gamma^\star(t)=6\alpha t+2\beta\) yields
\[
\mathcal A_T[\gamma^\star]
=12T^3\|\alpha\|_2^2
+12T^2\langle\alpha,\beta\rangle
+4T\|\beta\|_2^2.
\]
Substituting \(\alpha,\beta\) and completing the square gives
Eq.~\eqref{eqn_close_cost}.
\end{proof}

\subsection{Proof of Theorem~\ref{thm:terminal_velocity_projection}:
Variational Terminal-Velocity Projection}
\label{app:terminal_velocity_projection}

\begin{proof}
Completing the square in Eq.~\eqref{eqn_close_cost} gives
\[
\begin{aligned}
C_T(x_0,v_0,x_1,v_1)
&=\frac{3}{T^3}\|x_1-x_0-Tv_0\|_2^2\\
&\quad+\frac{4}{T}
\left\|v_1-\left(\frac{3}{2T}(x_1-x_0)-\frac12v_0\right)\right\|_2^2.
\end{aligned}
\]
Since \(T>0\), the second term has a unique minimizer in \(v_1\).
This proves the closure and projected cost in
Eqs.~\eqref{eqn_lat_terminal_v}--\eqref{eqn_proj_cost}.

To establish the coupling-level statement, write \(s=(x_0,v_0,x_1)\)
and let \(\eta\in\Pi(\mu_0,\rho_1)\) be an observable coupling.
Define its full-boundary lifts by
\[
\mathcal L(\eta)
=\{\pi\in\mathcal P(\mathcal X\times\mathcal V\times
\mathcal X\times\mathcal V):p_\#\pi=\eta\},
\qquad p(s,v_1)=s.
\]
For any \(\pi\in\mathcal L(\eta)\), the completed-square identity yields
\[
\int C_T(s,v_1)\,d\pi
=\int C_T^\star(s)\,d\eta
+\frac{4}{T}\int\|v_1-v_1^\star(s)\|_2^2\,d\pi
\geq\int C_T^\star(s)\,d\eta.
\]
Equality is attained by the deterministic lift
\(\pi^\star=(\mathrm{id},v_1^\star)_\#\eta\), which assigns
\(v_1=v_1^\star(s)\) to each observable tuple. Hence
\[
\inf_{\pi\in\mathcal L(\eta)}\int C_T\,d\pi
=\int C_T^\star\,d\eta.
\]
Taking the infimum over observable couplings gives
\[
\inf_{\eta\in\Pi(\mu_0,\rho_1)}\int C_T^\star\,d\eta
=\inf_{\eta\in\Pi(\mu_0,\rho_1)}
\inf_{\pi\in\mathcal L(\eta)}\int C_T\,d\pi,
\]
which proves the exact variational projection.
\end{proof}

\subsection{Mini-batch Pairing and Supervision Details}
\label{app:minibatch_pairing}

For the cost matrix \(C\) in Eq.~\eqref{eq:regime_cost_matrix}, we use
uniform minibatch weights \(a_i=1/m\) and \(b_j=1/n\). The entropic
coupling solves
\[
\min_{\substack{P\geq0\\P\mathbf1=a,\;P^\top\mathbf1=b}}
\sum_{i,j}P_{ij}C_{ij}
+\varepsilon\sum_{i,j}P_{ij}(\log P_{ij}-1).
\]
With \(K_{ij}=\exp(-C_{ij}/\varepsilon)\), Sinkhorn iterations give
\[
u\leftarrow\frac{a}{Kw},
\qquad
w\leftarrow\frac{b}{K^\top u},
\qquad
P^\star=\operatorname{diag}(u)K\operatorname{diag}(w),
\]
where divisions are elementwise. These updates may equivalently be
performed in the log domain for numerical stability.

Matched pairs are drawn jointly from \(P^\star\), or by sampling
\(i\sim a\) followed by
\[
j\sim q(\cdot\mid i),
\qquad
q(j\mid i)=\frac{P_{ij}^\star}{a_i}=mP_{ij}^\star.
\]
For each pair, Eq.~\eqref{eq:completed_terminal_velocity} supplies the
terminal velocity. Sampling \(t\sim\operatorname{Unif}(0,T)\) and
evaluating Eqs.~\eqref{eqn_hermite_bridge_x} and
\eqref{eqn_hermite_bridge_v} then gives
\((x_t^\star,v_t^\star,t/T,a_t^\star)\) for acceleration matching.
The complete training procedure is given in Algorithm~\ref{alg:vtvfm}.

\subsection{Local Integral Consistency}
\label{app:local_integral_consistency}

For a variational bridge and \(0\leq t\leq T-\Delta t\), its exact
increments satisfy
\[
\begin{aligned}
v_{t+\Delta t}^\star-v_t^\star
&=\int_t^{t+\Delta t}a_s^\star\,ds,\\
x_{t+\Delta t}^\star-x_t^\star-\Delta t\,v_t^\star
&=\int_t^{t+\Delta t}(t+\Delta t-s)a_s^\star\,ds.
\end{aligned}
\]
Approximating each integral using the learned acceleration at its
left endpoint gives
\[
\begin{aligned}
v_{t+\Delta t}^\star-v_t^\star
&\approx\Delta t\,a_\theta(x_t^\star,v_t^\star,t/T),\\
x_{t+\Delta t}^\star-x_t^\star-\Delta t\,v_t^\star
&\approx\frac12\Delta t^2a_\theta(x_t^\star,v_t^\star,t/T).
\end{aligned}
\]
The expected squared residuals are precisely
\(\mathcal L_{\mathrm{va}}\) and \(\mathcal L_{\mathrm{xv}}\) in
Eqs.~\eqref{eqn_loss_va}--\eqref{eqn_loss_xv}.
They supplement instantaneous acceleration matching with finite-step
increment supervision. Because the acceleration varies along a cubic
bridge, these are local approximations rather than exact integration
constraints. We use \(\Delta t=T/N\), as specified in
Appendix~\ref{app:implementation_details}.

\subsection{Acceleration Divergence Matching}
\label{app:acceleration_divergence_matching}

The phase-space drift is
\(b_\theta(x,v,t)=(v,a_\theta(x,v,t/T))\). Since \(v\) is an independent
phase-space coordinate, \(\nabla_x\cdot v=0\), and therefore
\[
\nabla_{(x,v)}\cdot b_\theta
=\nabla_v\cdot a_\theta.
\]
This identifies the velocity divergence of acceleration as the
compressibility term for the learned dynamics.

\paragraph{Path-error motivation.}
Let \(F_t^\star(x,v)=
\mathbb E[a_t^\star\mid x_t^\star=x,v_t^\star=v]\) be the marginal
reference acceleration field. Under sufficient regularity, the
reference and learned densities \(\mu_t^\star,\mu_t^\theta\) satisfy
their phase-space continuity equations. For
\(\eta_t=\mu_t^\theta-\mu_t^\star\), subtraction gives
\[
\partial_t\eta_t+\nabla_{(x,v)}\cdot(b_\theta\eta_t)
=\nabla_v\cdot\bigl[(F_t^\star-a_\theta)\mu_t^\star\bigr].
\]
The forcing term expands as
\[
(F_t^\star-a_\theta)\cdot\nabla_v\mu_t^\star
+\mu_t^\star
\bigl(\nabla_v\cdot F_t^\star-\nabla_v\cdot a_\theta\bigr).
\]
Thus, both acceleration and divergence mismatch contribute to the
path error, motivating divergence regularization alongside
acceleration matching~\citep{huang2025flowdivergence}.

\paragraph{Conditional continuation targets.}
Let \(\Delta=T-t>0\). Applying
Proposition~\ref{prop:full_boundary_bridge} and
Theorem~\ref{thm:terminal_velocity_projection} to the remaining interval
gives the current acceleration for a bridge from \((x,v)\) to position
\(y\):
\[
F_{\mathrm{free}}^\star(x,v,t;y)
=\frac{3}{\Delta^2}(y-x-\Delta v),
\]
when the terminal velocity is free. If the terminal velocity is fixed
at \(w\), the corresponding acceleration is
\[
F_{\mathrm{fix}}^\star(x,v,t;y,w)
=\frac{6}{\Delta^2}(y-x-\Delta v)-\frac{2}{\Delta}(w-v).
\]
Taking the velocity divergence with the terminal data held fixed gives
\[
\nabla_v\cdot F_{\mathrm{free}}^\star=-\frac{3d}{\Delta},
\qquad
\nabla_v\cdot F_{\mathrm{fix}}^\star=-\frac{4d}{\Delta}.
\]
Normalizing by \(\Delta/d\) yields the auxiliary objective
\[
\mathcal L_{\mathrm{ADM}}
=\mathbb E\left[
\left(
\frac{T-t}{d}\nabla_v\cdot
a_\theta(x_t^\star,v_t^\star,t/T)+c
\right)^2
\right],
\qquad
c=
\begin{cases}
3,&\text{latent }v_1,\\
4,&\text{observed }v_1.
\end{cases}
\]
The latent-\(v_1\) case is Eq.~\eqref{eqn_loss_adm}.

\paragraph{Conditional and marginal divergence.}
The constants above are exact for conditional continuation fields.
The marginal reference field averages over terminal data conditioned
on the current state. For example, in the latent-\(v_1\) case, let
\(m_t(x,v)=\mathbb E[x_1\mid x_t^\star=x,v_t^\star=v]\). Then
\[
F_t^\star(x,v)
=\frac{3}{\Delta^2}\bigl(m_t(x,v)-x-\Delta v\bigr),
\]
and, wherever differentiable,
\[
\nabla_v\cdot F_t^\star
=-\frac{3d}{\Delta}
+\frac{3}{\Delta^2}\nabla_v\cdot m_t(x,v).
\]
Accordingly, ADM regularizes the learned field toward the contraction
of the conditional bridge. Exact marginal-divergence matching would
also require the state dependence of the conditional terminal mean.
The same conditional--marginal distinction applies when terminal
velocities are observed.

\paragraph{Trace estimation.}
For low-dimensional data, we compute
\(\nabla_v\cdot a_\theta=\operatorname{Tr}(\partial a_\theta/\partial v)\)
directly by automatic differentiation. For high-dimensional data, we
use a Hutchinson estimator
\[
\nabla_v\cdot a_\theta
\approx
\epsilon^\top\frac{\partial a_\theta}{\partial v}\epsilon,
\qquad
\mathbb E[\epsilon\epsilon^\top]=I,
\]
with a standard Gaussian or Rademacher vector \(\epsilon\). The
quadratic form is evaluated using a vector--Jacobian product without
materializing the full Jacobian.

\subsection{Numerical Solvers}
\label{app:numerical_solvers}

Sampling integrates
\[
\dot z(t)=F_\theta(z(t),t),
\qquad
z=(x,v),
\qquad
F_\theta(z,t)=
\begin{pmatrix}
v\\a_\theta(x,v,t/T)
\end{pmatrix}.
\]
Let \(h=T/N\), \(t_k=kh\), and
\(a_k=a_\theta(x_k,v_k,t_k/T)\).

\paragraph{Heun solver.}
The predictor and corrected update are
\[
\widetilde z_{k+1}=z_k+hF_\theta(z_k,t_k),
\qquad
z_{k+1}=z_k+\frac{h}{2}
\left[
F_\theta(z_k,t_k)+F_\theta(\widetilde z_{k+1},t_{k+1})
\right].
\]

\paragraph{Taylor2 update.}
The update used in our comparisons is
\[
x_{k+1}=x_k+hv_k+\frac12h^2a_k,
\qquad
v_{k+1}=v_k+ha_k.
\]
Here \emph{Taylor2} denotes the second-order expansion of the position
update combined with an Euler velocity update. For a general learned
acceleration field, this scheme is not second-order accurate for the
full phase-space state.

\paragraph{Velocity Verlet update.}
We first predict the position and velocity:
\[
x_{k+1}=x_k+hv_k+\frac12h^2a_k,
\qquad
\overline v_{k+1}=v_k+ha_k.
\]
We then evaluate the acceleration at the predicted state and correct
the velocity:
\[
\widehat a_{k+1}
=a_\theta(x_{k+1},\overline v_{k+1},t_{k+1}/T),
\qquad
v_{k+1}=v_k+\frac{h}{2}(a_k+\widehat a_{k+1}).
\]
For velocity-dependent acceleration, this is an explicit
predictor--corrector variant of the Verlet update. At the next step,
\(a_{k+1}\) is evaluated at the corrected state
\((x_{k+1},v_{k+1})\).

\vspace{-1pt}
\section{Implementation Details}
\label{app:implementation_details}
\vspace{-1pt}

This appendix summarizes the implementation and evaluation settings
for the low-dimensional, PDE, and CIFAR-10 experiments.
We report both method-specific and solver-controlled inference
protocols, with solver choices and sampling budgets specified
for each experiment.

\subsection{Low-Dimensional Experiments}
\label{app:impl_toy}

For the low-dimensional datasets, all methods are evaluated with the same sampling budget of 100 steps and EMA weights.
For each dataset, we use \(n=2000\) target samples and \(n=2000\) generated samples to compute the sample-based \(W_2\) metric.
CFM, OT-CFM, OAT-FM, and CurveFlow are evaluated with Euler sampling, following their standard first-order sampling protocol.
HOFM is evaluated with the Taylor2 solver, following its original high-order setting.
For VTV-FM, we evaluate Heun, Taylor2, and Velocity Verlet, and report the best validation result in the main table.
The solver-specific results are reported separately in the ablation table.

\subsection{Shared Source Phase-Space and Training Setup}

\paragraph{Source phase-space prior.}
For the low-dimensional and CIFAR-10 experiments, we sample
\[
v_0 \sim \mathcal N(0,\sigma^2 I),
\]
independently of the source position \(x_0\).
Unless otherwise stated, we use \(\sigma=1\).
For static data, \(v_0\) is a modeling prior used to provide a
non-degenerate lifting from position space to phase space; it is not
interpreted as an observed or physical velocity.
The variational terminal-velocity closure itself does not depend on the
Gaussian assumption.

The resulting projected cost is
\[
C_T^\star(x_0,v_0,x_1)
=
\frac{3}{T^3}
\left\|
x_1-x_0-Tv_0
\right\|_2^2.
\]
When \(v_0=0\), this cost is proportional to the standard quadratic
OT cost, so their unregularized optimal couplings coincide.
The induced position trajectory follows a cubic time parameterization
rather than a constant-velocity linear bridge.
The exact constant-velocity trajectory is recovered only when
\(v_0=(x_1-x_0)/T\), for which the residual
\(x_1-x_0-Tv_0\) and the required acceleration vanish.

\paragraph{Loss coefficients and local consistency window.}
The full training objective is
\[
\mathcal L_{\mathrm{total}}
=
\mathcal L_{\mathrm{acc}}
+
\lambda_{\mathrm{va}}\mathcal L_{\mathrm{va}}
+
\lambda_{\mathrm{xv}}\mathcal L_{\mathrm{xv}}
+
\lambda_{\mathrm{ADM}}\mathcal L_{\mathrm{ADM}}.
\]
Unless otherwise stated, we use
\[
\lambda_{\mathrm{va}}
=
\lambda_{\mathrm{xv}}
=
\lambda_{\mathrm{ADM}}
=
0.1,
\]
with coefficient \(1\) on \(\mathcal L_{\mathrm{acc}}\).
For the two local consistency losses, the local time window is matched
to the actual fixed sampling step,
\[
\Delta t = \frac{T}{N},
\]
and is not tuned as an additional hyperparameter.

\subsection{PDE-Governed Field Experiments}
\label{app:impl_pde}

For PDE-governed field generation, all methods are evaluated with 100 sampling steps.
Unlike static image data, PDE data provide terminal derivative information through the PDE right-hand side or simulator outputs.
Therefore, VTV-FM directly uses the observed or derivable terminal velocity and does not apply terminal-velocity closure in this setting.
This corresponds to the observed/derivable-\(v_1\) regime described in the main text.

The baseline solver choices follow the same principle as in the low-dimensional experiments.
First-order FM variants use their standard first-order samplers, HOFM uses the Taylor2 solver, and VTV-FM is evaluated with phase-space solvers such as Heun, Taylor2, and Velocity Verlet.
Other evaluation settings, including sample size and metric computation, follow the low-dimensional experiment protocol unless otherwise specified.

\subsection{CIFAR-10 Experiments}
\label{app:impl_cifar10}

CIFAR-10 represents the high-dimensional missing-\(v_1\) regime,
where only terminal images are observed and target-side terminal
derivatives are unavailable.
Therefore, VTV-FM uses the variational terminal-velocity closure derived
in Section~\ref{sec:method} to instantiate the second-order bridge.

For the controlled cross-method comparison in
Table~\ref{tab:cifar10_solver_controlled}, we evaluate FM, I-CFM,
OT-CFM, OAT-FM, CurveFlow, and VTV-FM under the applicable shared
inference protocols.
VTV-FM uses acceleration matching only in this comparison, without the
local consistency or ADM regularizers.
Heun-100 provides the primary fixed-step solver-controlled comparison,
while Dopri5 provides a shared adaptive-solver comparison with measured
NFE.

Table~\ref{tab:cifar10_ablation} separately evaluates the effect of the
auxiliary objectives within VTV-FM under Velocity Verlet, Heun, and
Taylor2.
This separates the cross-method comparison from the contribution of
local consistency and acceleration-divergence matching.

\begin{figure*}[t]
    \centering
    \includegraphics[width=0.88\textwidth]{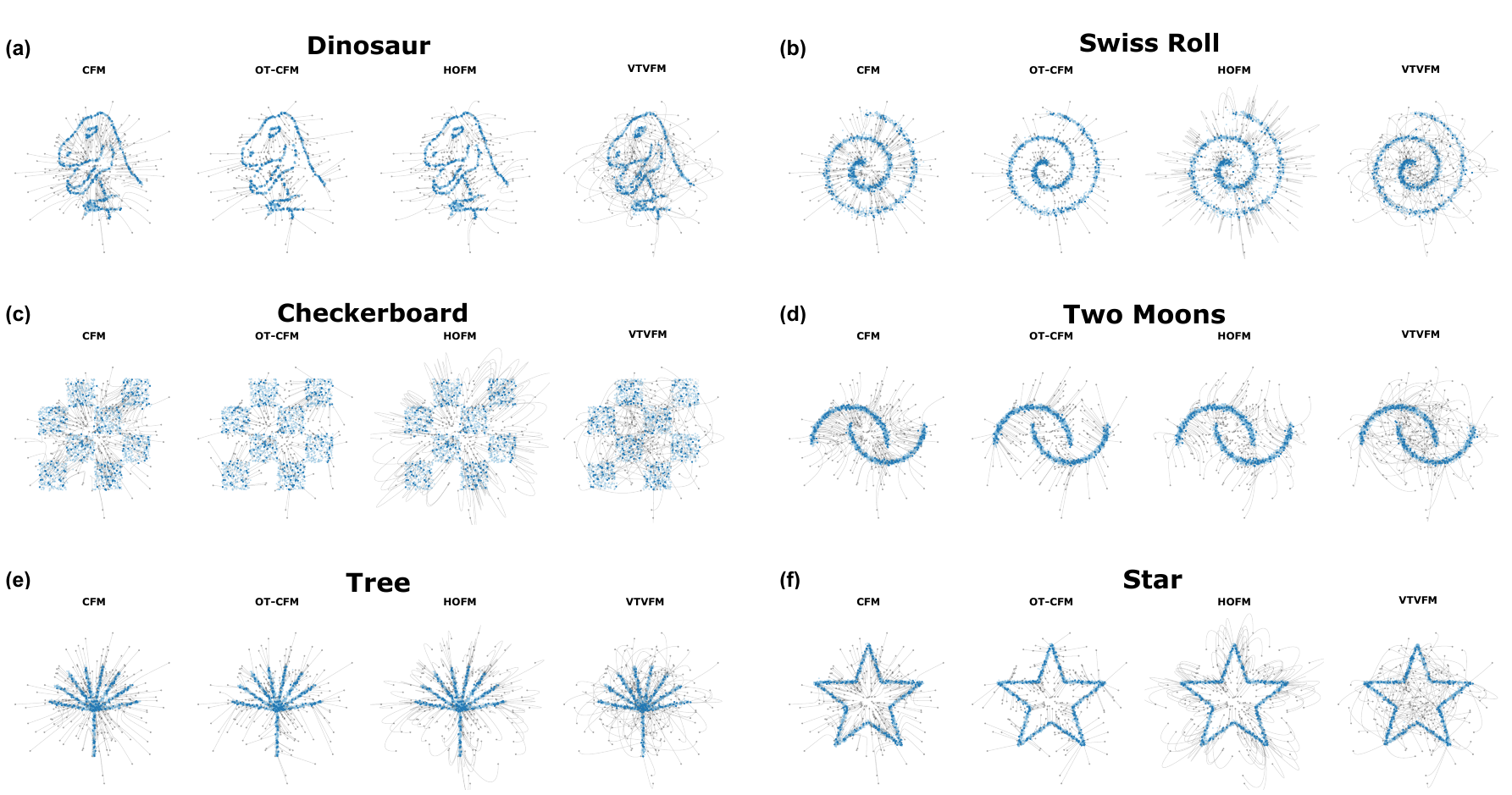}
    \caption{Transport trajectories on six low-dimensional datasets.
    Each panel compares \textsc{CFM}, \textsc{OT-CFM}, \textsc{HOFM},
    and \textsc{VTV-FM}.}
    \label{fig:toy_trajectory_grid}
\end{figure*}

\section{Additional Qualitative Visualizations}
\label{app:visualizations}

This appendix complements the quantitative results in
Section~\ref{sec:experiments} with transport trajectories and generated
samples.

\subsection{Low-Dimensional Trajectory Visualizations}
\label{app:toy_vis}

Figure~\ref{fig:toy_trajectory_grid} compares transport geometry on six
representative datasets, complementing the sample-based \(W_2\) results
in Section~\ref{sec:toy}.

\begin{figure}[!th]
    \centering
    \begin{minipage}{0.4\linewidth}
        \centering
        \includegraphics[width=\linewidth]{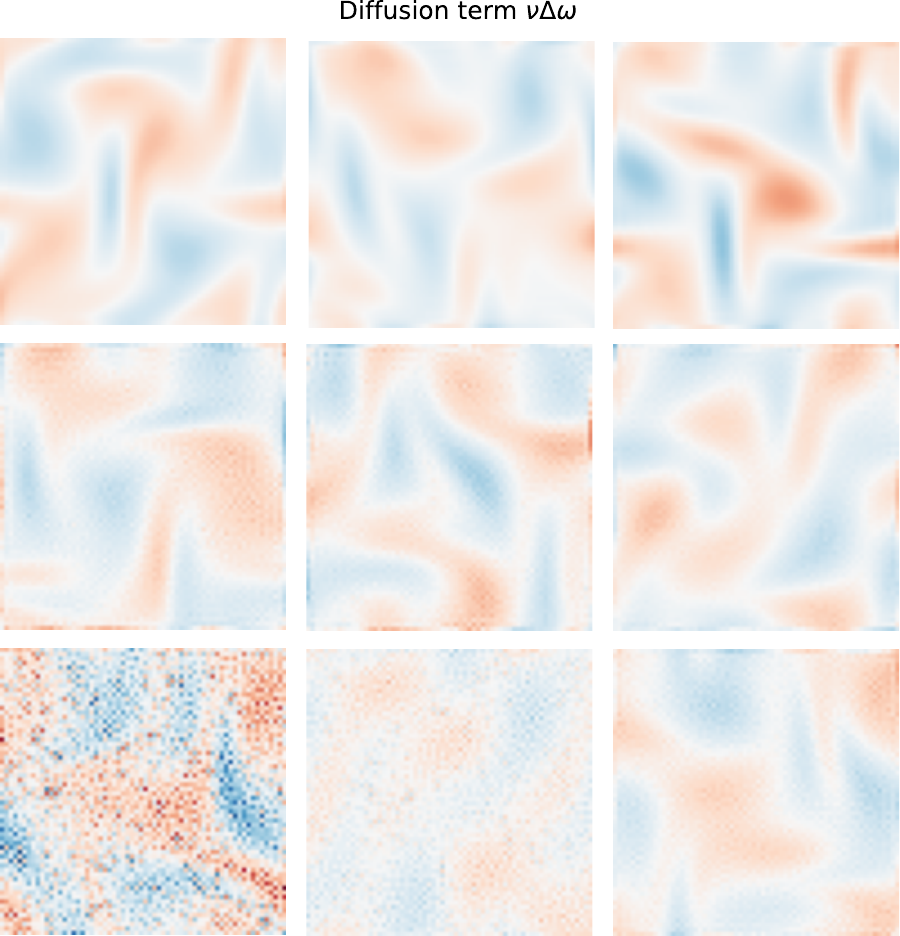}\\
        \small NS2D sample 1
    \end{minipage}
    \hfill
    \begin{minipage}{0.4\linewidth}
        \centering
        \includegraphics[width=\linewidth]{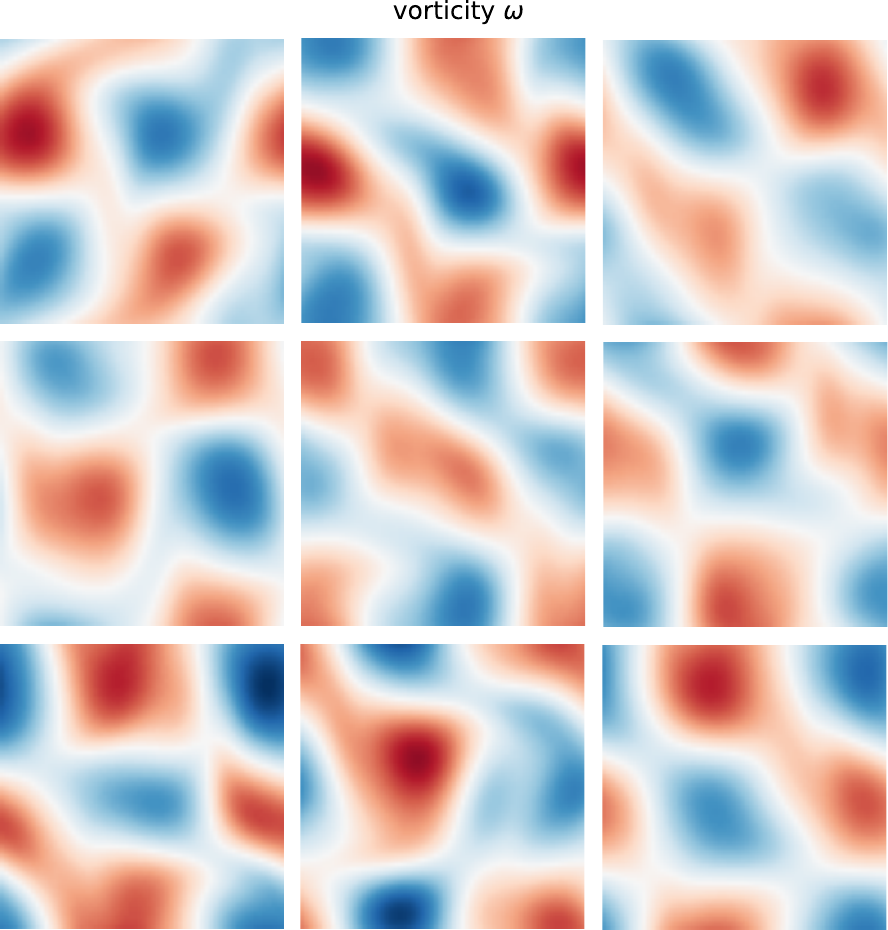}\\
        \small NS2D sample 2
    \end{minipage}

    \vspace{0.6em}
    \begin{minipage}{0.4\linewidth}
        \centering
        \includegraphics[width=\linewidth]{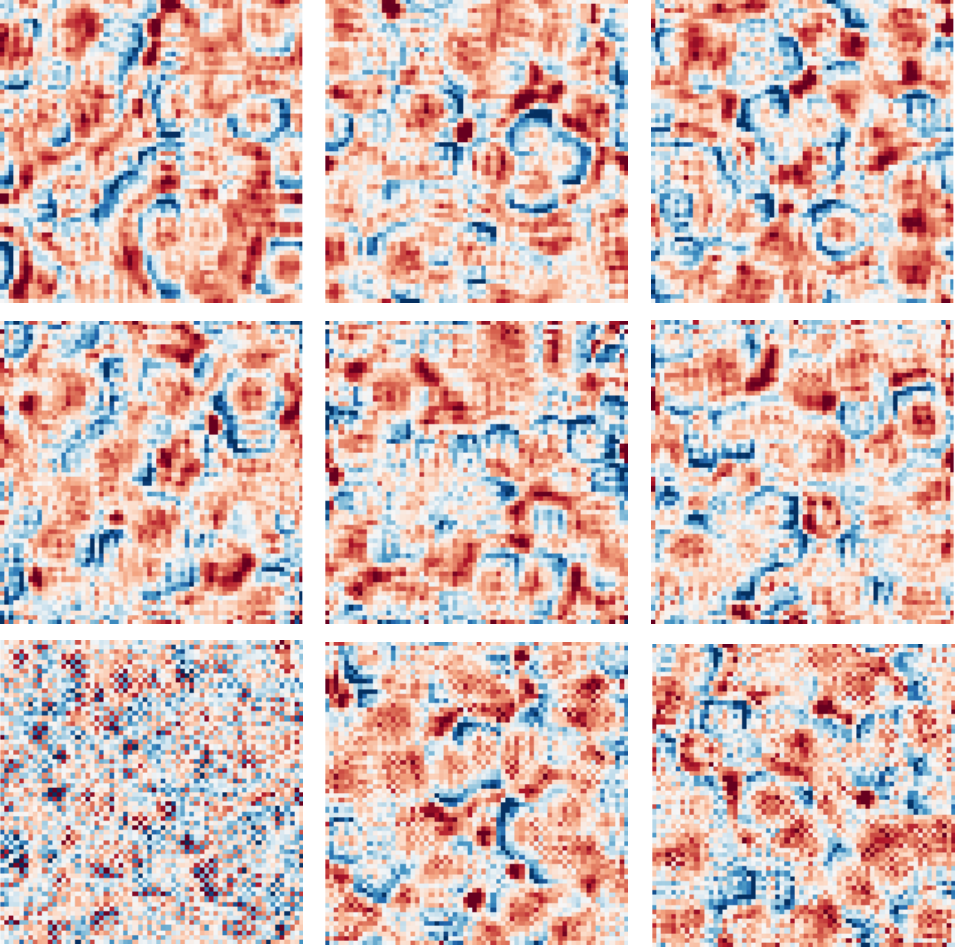}\\
        \small KS2D sample 1
    \end{minipage}
    \hfill
    \begin{minipage}{0.4\linewidth}
        \centering
        \includegraphics[width=\linewidth]{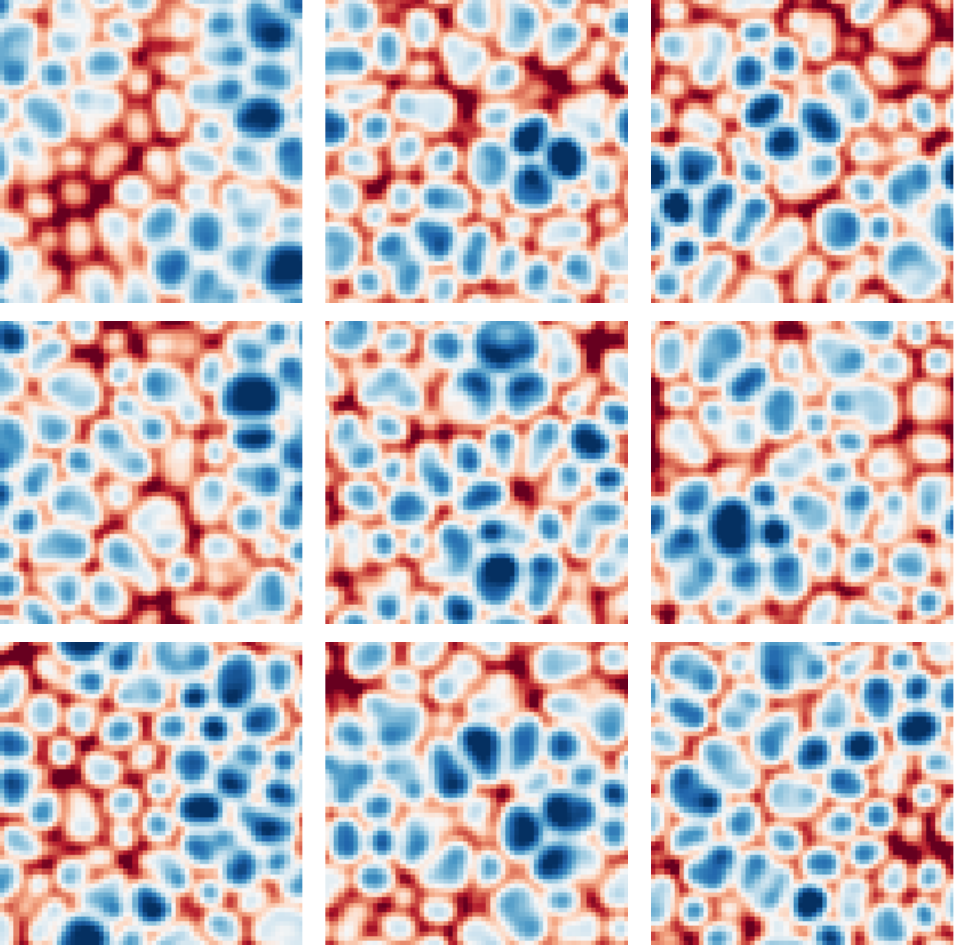}\\
        \small KS2D sample 2
    \end{minipage}
    \caption{NS2D and KS2D generation examples. Within each example,
    the top, middle, and bottom rows show real, VTV-FM, and OT-CFM
    samples, respectively.}
    \label{fig:pde_vis}
\end{figure}

\subsection{Additional PDE Visualization}
\label{app:pde_vis}

Figure~\ref{fig:pde_vis} compares real and generated fields for the
NS2D and KS2D experiments in Section~\ref{sec:pde}.

\subsection{Additional CIFAR-10 Samples}
\label{app:cifar10_samples}

Figure~\ref{fig:cifar10_samples} uses the best VTV-FM configuration in
Table~\ref{tab:cifar10_ablation}: the full objective, variational
terminal-velocity closure, and 100-step Velocity Verlet sampling.

\begin{figure}[!th]
    \centering
    \includegraphics[width=0.40\linewidth]{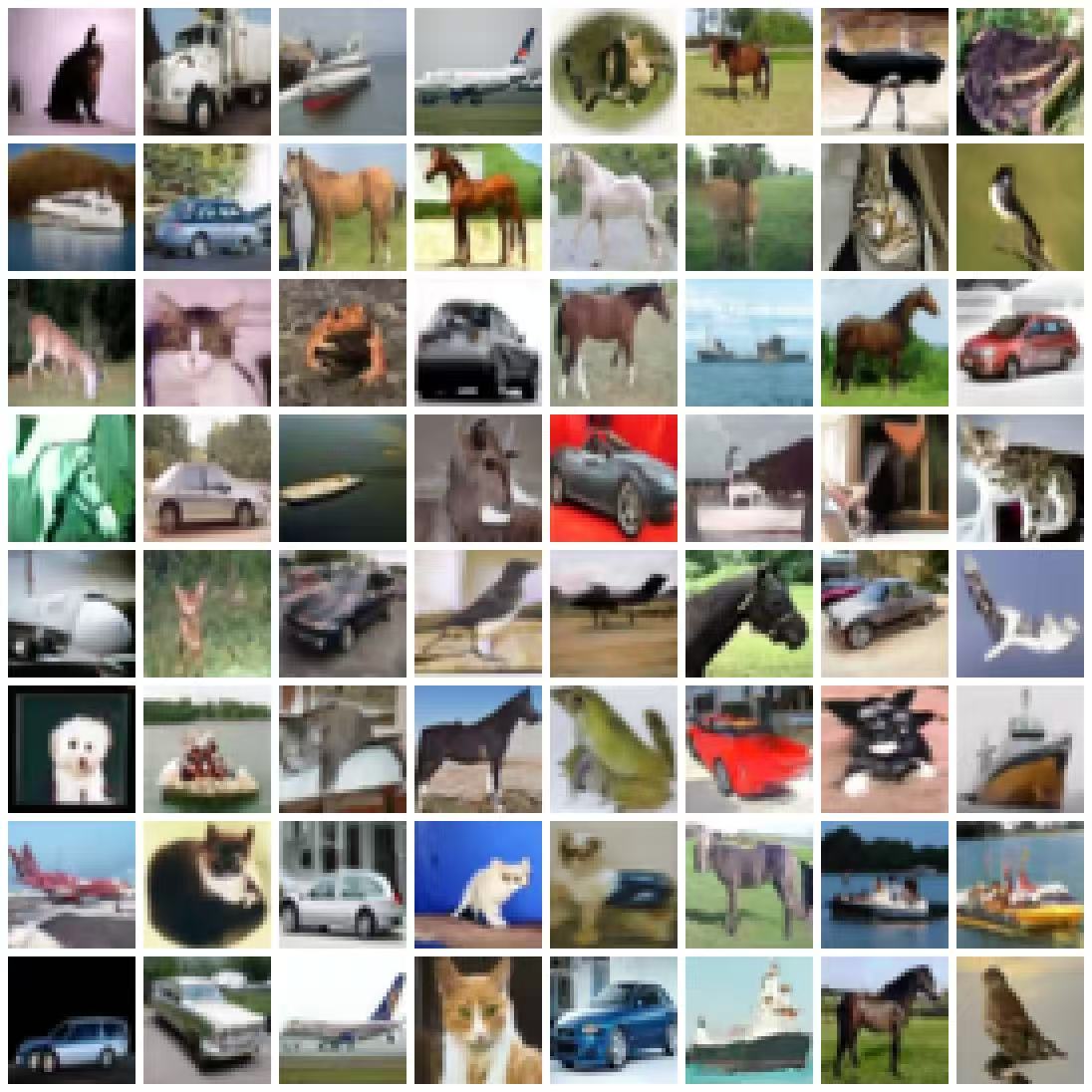}
    \caption{Qualitative CIFAR-10 samples generated by VTV-FM.}
    \label{fig:cifar10_samples}
\end{figure}

\subsection{Additional Video Visualization}
\label{app:video_vis}

Figure~\ref{fig:video_vis} provides qualitative examples on KTH, BAIR,
and Navier--Stokes temporal sequences. Terminal velocities are
approximated by finite differences between adjacent frames and used in
the full-boundary bridge. These examples illustrate the
observed/derivable-\(v_1\) construction on temporal data; they do not
constitute a quantitative video-generation benchmark.

\begin{figure}[!th]
    \centering
    \begin{minipage}{0.45\linewidth}
        \centering
        \includegraphics[width=\linewidth]{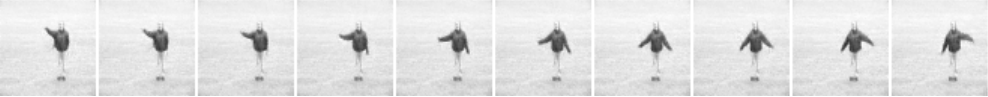}\\[-0.1em]
        {\small KTH sample 1}
    \end{minipage}
    \hfill
    \begin{minipage}{0.45\linewidth}
        \centering
        \includegraphics[width=\linewidth]{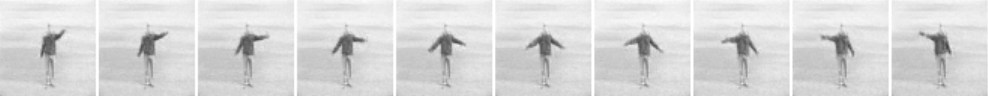}\\[-0.1em]
        {\small KTH sample 2}
    \end{minipage}

    \vspace{0.6em}
    \begin{minipage}{0.45\linewidth}
        \centering
        \includegraphics[width=\linewidth]{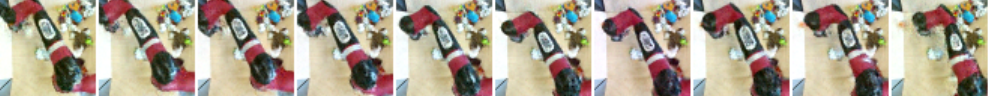}\\[-0.1em]
        {\small BAIR sample 1}
    \end{minipage}
    \hfill
    \begin{minipage}{0.45\linewidth}
        \centering
        \includegraphics[width=\linewidth]{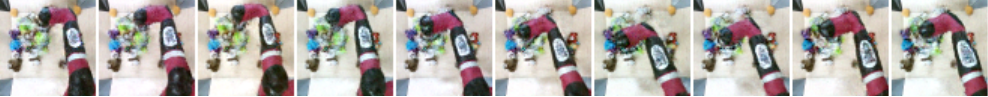}\\[-0.1em]
        {\small BAIR sample 2}
    \end{minipage}

    \vspace{0.6em}
    \begin{minipage}{0.48\linewidth}
        \centering
        \includegraphics[width=\linewidth,trim=0pt 0pt 0pt 22pt,clip]
        {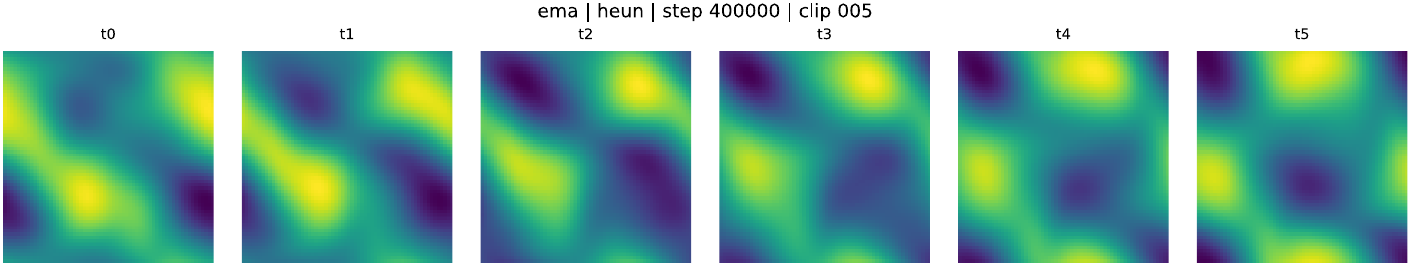}\\[-0.1em]
        {\small Navier--Stokes sample 1}
    \end{minipage}
    \hfill
    \begin{minipage}{0.48\linewidth}
        \centering
        \includegraphics[width=\linewidth,trim=0pt 0pt 0pt 22pt,clip]
        {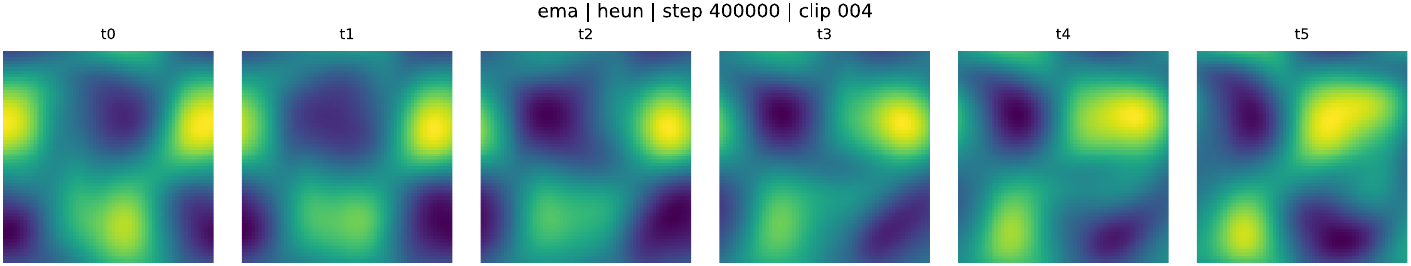}\\[-0.1em]
        {\small Navier--Stokes sample 2}
    \end{minipage}
    \caption{Qualitative temporal examples from KTH, BAIR, and
    Navier--Stokes sequences.}
    \label{fig:video_vis}
\end{figure}

\section{Additional Controlled Experiments}
\label{app:controlled_experiments}

This appendix provides additional controlled experiments that complement
the main results.
We examine loss and solver analysis on low-dimensional datasets, the source phase-space prior, terminal-velocity completion,
the consistency between endpoint pairing and bridge construction, and
computational efficiency and scalability.
We do not repeat the solver-controlled CIFAR-10 comparisons already
reported in Section~\ref{sec:cifar10}.

\subsection{Loss and Solver Analysis on Low-Dimensional Datasets}
\label{app:toy_ablation_solver}
\vspace{-7pt}
We provide additional ablation and solver analysis for VTV-FM on the low-dimensional datasets.
Table~\ref{tab:toy_ablation_solver} summarizes two studies.
First, we compare different training objectives: acceleration-only training, acceleration with local consistency, acceleration with acceleration-divergence matching (ADM), and a velocity--acceleration joint variant.
Second, we compare three fixed-step solvers: Heun, Taylor2, and Velocity Verlet, using the best-performing loss setting for each dataset.

The acceleration-only objective is already competitive, achieving an average sample-based \(W_2\) of \(0.1018\).
Adding local consistency gives the best average among the loss variants, reducing the mean \(W_2\) to \(0.1012\), while ADM gives a similar average of \(0.1013\).
This suggests that the main gain comes from the bridge-induced acceleration supervision, with auxiliary terms providing modest refinements.
For solver choice, all three fixed-step solvers perform similarly under the same sampling budget.
Heun obtains the best average \(W_2\) of \(0.1007\), while Taylor2 and Velocity Verlet remain close at \(0.1011\) and \(0.1013\), respectively.
These results indicate that VTV-FM is not tied to a particular solver; its gains mainly come from the second-order phase-space transport construction.

\begin{table}[h]
\centering
\caption{
Ablation and fixed-step solver analysis for VTV-FM on low-dimensional datasets.
Left: loss variants. Right: solver comparison under the best-performing loss setting for each dataset.
}
\label{tab:toy_ablation_solver}
\small
\setlength{\tabcolsep}{3.4pt}
\begin{tabular}{@{}lccccccc@{}}
\toprule
& \multicolumn{4}{c}{Loss ablation} 
& \multicolumn{3}{c}{Solver comparison} \\
\cmidrule(lr){2-5}
\cmidrule(l){6-8}
Dataset 
& Acc. 
& +Cons. 
& +ADM 
& Vel.+Acc. 
& Heun 
& Taylor2 
& Verlet \\
\midrule
Dinosaur 
& 0.1119 
& 0.1114 
& 0.1112 
& \textbf{0.1109} 
& 0.1113 
& 0.1121 
& \textbf{0.1109} \\
Swiss 
& 0.1112 
& \textbf{0.1081} 
& 0.1105 
& 0.1127 
& \textbf{0.1081} 
& 0.1092 
& 0.1105 \\
Checker 
& \textbf{0.1092} 
& 0.1094 
& 0.1098 
& 0.1092 
& \textbf{0.1092} 
& 0.1096 
& 0.1099 \\
Moons 
& 0.0835 
& 0.0826 
& 0.0807 
& \textbf{0.0790} 
& \textbf{0.0790} 
& 0.0803 
& 0.0808 \\
Lines 
& \textbf{0.1362} 
& 0.1366 
& 0.1380 
& 0.1387 
& 0.1370 
& \textbf{0.1362} 
& 0.1376 \\
Tree 
& 0.0741 
& 0.0739 
& \textbf{0.0730} 
& 0.0752 
& 0.0736 
& 0.0734 
& \textbf{0.0730} \\
Star 
& 0.0868 
& 0.0866 
& \textbf{0.0862} 
& 0.0877 
& 0.0869 
& 0.0867 
& \textbf{0.0862} \\
\midrule
Avg.
& 0.1018
& \textbf{0.1012}
& 0.1013
& 0.1019
& \textbf{0.1007}
& 0.1011
& 0.1013 \\
\bottomrule
\end{tabular}
\end{table}
\vspace{-7pt}

% ============================================================
\subsection{Source Phase-Space Prior and Sensitivity Analysis}
\label{app:source_prior_ablation}

For static datasets, we introduce a source velocity
\[
v_0 \sim \mathcal N(0,\sigma^2 I),
\]
sampled independently of the source position \(x_0\).
Unless otherwise stated, we use \(\sigma=1\).

The role of \(v_0\) is to define the source phase-space state required
by the second-order formulation.
For static data, it is a modeling variable rather than an observed
physical velocity, and its scale is not intended to be a finely tuned
performance parameter.

To evaluate sensitivity to this choice, we vary
\(\sigma\) while keeping the remaining training and evaluation settings
fixed.
Table~\ref{tab:v0_scale} reports the resulting average \(W_2\) over the
seven low-dimensional datasets.

\vspace{-7pt}
\begin{table}[!htbp]
\centering
\caption{
Sensitivity to the source-velocity scale, averaged over the seven
low-dimensional datasets.
}
\label{tab:v0_scale}
\small
\setlength{\tabcolsep}{5pt}
\renewcommand{\arraystretch}{0.95}
\begin{tabular}{c|ccccc}
\toprule
\(\sigma\) & 0 & 0.5 & 1.0 & 1.5 & 2.0 \\
\midrule
Avg. \(W_2\)
& 0.1283
& \textbf{0.1002}
& 0.1004
& 0.1011
& 0.1056 \\
\bottomrule
\end{tabular}
\end{table}

As shown in Table~\ref{tab:v0_scale}, performance remains similar over
the moderate nonzero range \(\sigma\in[0.5,1.5]\), indicating limited
sensitivity to the precise source-velocity scale.
In contrast, \(\sigma=0\) collapses the velocity variation in the source
phase-space distribution and gives substantially worse performance.
Thus, \(v_0\) primarily provides the source phase-space representation,
while the precise choice of a reasonable nonzero scale has only a
modest effect on performance.

% ============================================================
\subsection{Terminal-Velocity Ablation Study}
\label{app:terminal_velocity_ablation}

We directly compare the proposed variational terminal-velocity closure
\[
v_1^\star
=
\frac{3}{2T}(x_1-x_0)-\frac{1}{2}v_0
\]
with five alternative completion rules:
\emph{Zero}, \(v_1=0\);
\emph{Chord}, \(v_1=(x_1-x_0)/T\);
\emph{Random}, which uses a random direction with the same norm as
\(v_1^\star\);
\emph{Copy-\(v_0\)}, \(v_1=v_0\); and
\emph{Tangent}, which sets \(v_1\) along the local target-data tangent
estimated by one-dimensional PCA around \(x_1\), with its sign aligned
to \(v_1^\star\) and its norm matched to \(\|v_1^\star\|_2\).

We consider two settings.
In the \emph{fixed-pairing} setting, the projected-2OT source--target
pairs are held fixed and only the terminal-velocity rule is changed.
This isolates the effect of terminal-velocity completion and the
resulting bridge supervision.
In the \emph{end-to-end} setting, each alternative terminal-velocity
rule is substituted into the full-boundary cost, after which pairing
and training are recomputed; VTV-FM uses its projected cost and the
corresponding variational bridge.
Table~\ref{tab:v1_ablation} summarizes both settings, averaged over the
seven low-dimensional datasets.
\vspace{-7pt}
\begin{table}[!htbp]
\centering
\caption{
Terminal-velocity ablation averaged over the seven low-dimensional
datasets.
}
\label{tab:v1_ablation}
\small
\setlength{\tabcolsep}{7pt}
\renewcommand{\arraystretch}{0.95}
\begin{tabular}{lcccccc}
\toprule
Setting
& \(v_1^\star\)
& Zero
& Chord
& Random
& Copy-\(v_0\)
& Tangent \\
\midrule
Fixed pairing
& \textbf{0.1004}
& 0.1050
& 0.1020
& 0.1094
& 0.1062
& 0.1013 \\
End-to-end
& \textbf{0.1004}
& 0.1041
& 0.1008
& 0.1087
& 0.1054
& 0.1005 \\
\bottomrule
\end{tabular}
\end{table}
\vspace{-7pt}
As shown in Table~\ref{tab:v1_ablation}, with pairing fixed, the proposed
closure gives the lowest average error among all terminal-velocity rules.
The advantage remains in the end-to-end setting, where pairing and
training are recomputed for each completion rule.
These results show that the variational terminal-velocity completion
provides a robust bridge construction beyond simply changing the
endpoint pairing.

% ============================================================
\subsection{Pairing--Bridge Cross-Ablation Study}
\label{app:pairing_bridge_cross}

We further separate the effects of endpoint pairing and bridge
construction by crossing standard and projected OT pairing with
first- and second-order bridges.
Table~\ref{tab:pairing_bridge_cross} reports the four resulting
combinations.

\begin{table}[!htbp]
\centering
\caption{
Pairing--bridge cross-ablation on the low-dimensional datasets.
}
\label{tab:pairing_bridge_cross}
\small
\setlength{\tabcolsep}{4.5pt}
\renewcommand{\arraystretch}{0.95}
\begin{tabular}{lccc}
\toprule
Method & Bridge & Pairing & Avg. \(W_2\) \\
\midrule
OT-CFM
& Linear
& Standard OT
& 0.1077 \\
2OT-CFM
& Linear
& Projected 2OT
& 0.1109 \\
OT-VTV-FM
& Second-order
& Standard OT
& 0.1283 \\
2OT-VTV-FM
& Second-order
& Projected 2OT
& \textbf{0.1004} \\
\bottomrule
\end{tabular}
\end{table}

As shown in Table~\ref{tab:pairing_bridge_cross}, neither mismatched
cross-combination improves performance.
Projected 2OT does not improve the linear bridge, while combining the
second-order bridge with standard OT also degrades performance.
The strongest result is obtained when the projected pairing cost and
second-order bridge are used together.
This supports the interpretation that the improvement is associated
with the variational consistency among terminal-velocity closure,
pairing cost, and bridge construction, rather than with replacing any
single component in isolation.

% ============================================================
\subsection{Computational Efficiency and ImageNet-256 Scaling}
\label{app:efficiency_scaling}

\paragraph{Computational efficiency on low-dimensional benchmarks.}
We measure wall-clock inference time and peak allocated GPU memory using
the same hardware, evaluation batch size, and solver configuration.
VTV-FM uses acceleration matching only in this comparison.
Table~\ref{tab:compute_matched} reports the average \(W_2\), number of
function evaluations, wall-clock inference time, and peak GPU memory
on the seven low-dimensional benchmarks.

\begin{table}[!htbp]
\centering
\caption{
Compute-matched comparison on the low-dimensional benchmarks.
}
\label{tab:compute_matched}
\small
\setlength{\tabcolsep}{6pt}
\renewcommand{\arraystretch}{0.95}
\begin{tabular}{llcccc}
\toprule
Solver
& Method
& Avg. \(W_2\)
& NFE
& Time (ms)
& Peak memory (MB) \\
\midrule
Heun-100
& CFM
& 0.1085 & 200 & 858 & 62.30 \\
& OT-CFM
& 0.1079 & 200 & 861 & 62.30 \\
& VTV-FM
& \textbf{0.1020} & 200 & 859 & 62.37 \\
\midrule
Dopri5
& CFM
& 0.1069 & 72.8 & 486 & 62.58 \\
& OT-CFM
& 0.1066 & 42.2 & 306 & 62.59 \\
& VTV-FM
& \textbf{0.1018} & \textbf{38.0} & \textbf{242} & 63.66 \\
\midrule
Verlet-100
& VTV-FM
& \textbf{0.1009} & 200 & 687 & 62.96 \\
\bottomrule
\end{tabular}
\end{table}
\vspace{-7pt}
As shown in Table~\ref{tab:compute_matched}, under Heun-100, VTV-FM
improves average \(W_2\) with essentially unchanged wall-clock time and
peak memory relative to the first-order baselines.
Under Dopri5, it also achieves the lowest error with fewer function
evaluations and lower inference time.
Although VTV-FM augments the dynamical state from \(x\) to \((x,v)\),
the measured increase in peak memory remains small.

\paragraph{ImageNet-256 scaling.}
We additionally evaluate VTV-FM on ImageNet-256 to test scalability
beyond CIFAR-10.
We fine-tune the official SiT checkpoint~\citep{ma2024sit} and otherwise
follow the OAT-FM evaluation setting~\citep{yue2025oatfm}, including
CFG \(=1.5\), SDE sampling, and the official SiT FID evaluation protocol.
VTV-FM obtains an FID of \(1.97\), compared with \(2.06\) for SiT and
\(2.00\) for OAT-FM.
We primarily interpret this experiment as evidence that the proposed
construction scales beyond CIFAR-10, rather than as evidence that
minimum acceleration is universally superior to other high-order
objectives.

\section{VTV-FM as an Extension of OT-CFM}
\label{app:projected_kinetic_ot}

We relate the construction in Section~\ref{sec:method} to first-order
and kinetic optimal transport, then compare the minimum-acceleration
bridge with a minimum-jerk alternative.

\subsection{OT-CFM as First-Order Variational Flow Matching}
\label{app:first_order_otcfm}

For fixed endpoints, the least kinetic action is
\[
\inf_{\substack{\gamma(0)=x_0\\\gamma(T)=x_1}}
\int_0^T\|\dot\gamma(t)\|_2^2\,dt
=\frac{1}{T}\|x_1-x_0\|_2^2,
\]
attained by the linear bridge in Section~\ref{sec:preliminaries}.
At the distribution level, the Benamou--Brenier
formulation~\citep{benamou2000computational} gives
\[
\frac{1}{T}W_2^2(\rho_0,\rho_1)
=\inf_{\rho_t,u_t}
\int_0^T\int\|u_t(x)\|_2^2\,d\rho_t(x)\,dt,
\]
subject to
\[
\partial_t\rho_t+\nabla_x\cdot(\rho_tu_t)=0,
\qquad \rho_{t=0}=\rho_0,\quad\rho_{t=T}=\rho_1.
\]
Thus, the quadratic pairing cost and constant conditional velocity
target arise from the same first-order action.

\subsection{Second-Order Lift to Phase Space}
\label{app:second_order_kinetic_lift}

For phase-space laws
\(\mu,\nu\in\mathcal P_2(\mathbb R^d\times\mathbb R^d)\), the analogous
kinetic transport problem minimizes acceleration energy:
\[
\mathsf d_T^2(\mu,\nu)
=T\inf_{\mu_t,F_t}
\int_0^T\int\|F_t(x,v)\|_2^2\,d\mu_t(x,v)\,dt,
\]
subject to
\[
\partial_t\mu_t+v\cdot\nabla_x\mu_t+\nabla_v\cdot(F_t\mu_t)=0,
\qquad \mu_{t=0}=\mu,\quad\mu_{t=T}=\nu.
\]
Its static formulation~\citep{brigati2025otikin} is
\[
\mathsf d_T^2(\mu,\nu)
=\inf_{\pi\in\Pi(\mu,\nu)}
\int c_T(x,v,y,w)\,d\pi,
\qquad c_T=T\,C_T.
\]
Here \(C_T\) is the full-boundary cost in Eq.~\eqref{eqn_close_cost};
the factor \(T\) accounts for the normalization used above.
The minimizing bridge and its derivatives are given in
Eqs.~\eqref{eqn_hermite_bridge_x}--\eqref{eqn_hermite_bridge_v};
Appendix~\ref{app:full_boundary_bridge} proves
Proposition~\ref{prop:full_boundary_bridge}.

\subsection{Projection under Missing Terminal Velocity}
\label{app:projected_kinetic_cost}

When only the target position law \(\rho_1\) is prescribed, the target
velocity is optimized out. Theorem~\ref{thm:terminal_velocity_projection}
and Appendix~\ref{app:terminal_velocity_projection} show that the
resulting observable transport problem is
\[
\inf_{\eta\in\Pi(\mu_0,\rho_1)}
\int C_T^\star(x_0,v_0,x_1)\,d\eta.
\]
The minimizing lift assigns \(v_1=v_1^\star(x_0,v_0,x_1)\), with the
closure and projected cost given in
Eqs.~\eqref{eqn_lat_terminal_v}--\eqref{eqn_proj_cost}.
Their use in pairing and supervision is described in
Section~\ref{sec:minibatch_pairing} and
Appendix~\ref{app:minibatch_pairing}.
This gives the following correspondence:

\begin{center}
\small
\begin{tabular}{lccc}
\toprule
Method & Action & Endpoint cost & Conditional bridge\\
\midrule
OT-CFM
& \(\int_0^T\|\dot\gamma\|_2^2\,dt\)
& \(\|x_1-x_0\|_2^2/T\)
& Linear\\
VTV-FM
& \(\int_0^T\|\ddot\gamma\|_2^2\,dt\)
& \(C_T^\star(x_0,v_0,x_1)\)
& Projected cubic\\
\bottomrule
\end{tabular}
\end{center}

Both constructions derive the pairing cost and supervision from one
action; the second-order construction additionally determines the
unobserved terminal velocity.

\subsection{Bridge-Shape Ablation: Minimum Acceleration versus Minimum Jerk}
\label{app:min_jerk_ablation}

To assess the role of the cubic bridge, we replace the acceleration
action by
\[
\mathcal J_T[\gamma]=\int_0^T\|\gamma^{(3)}(t)\|_2^2\,dt.
\]
The alternative fixes \(\gamma(0)=x_0\), \(\dot\gamma(0)=v_0\),
and \(\gamma(T)=x_1\), imposes
\(\ddot\gamma(0)=\ddot\gamma(T)=0\), and leaves the terminal velocity
free. We call it \textsc{VTV-FM-Jerk-ZA}.
The Euler--Lagrange equation \(\gamma^{(6)}=0\), together with the
natural condition \(\gamma^{(4)}(T)=0\), gives the projected quintic
bridge below.

Let \(\tau=t/T\) and \(\delta=x_1-x_0-Tv_0\). Both bridges have the form
\[
x_q(\tau)=x_0+T\tau v_0+q(\tau)\delta,
\qquad
v_q(\tau)=v_0+\frac{q'(\tau)}{T}\delta,
\qquad
a_q(\tau)=\frac{q''(\tau)}{T^2}\delta,
\]
with
\[
q_{\mathrm{acc}}(\tau)=\frac32\tau^2-\frac12\tau^3,
\qquad
q_{\mathrm{jerk}}(\tau)=\frac52\tau^3-\frac{15}{8}\tau^4
+\frac38\tau^5.
\]
Here \(q_{\mathrm{acc}}\) follows by substituting
Eq.~\eqref{eqn_lat_terminal_v} into
Eq.~\eqref{eqn_hermite_bridge_x}.
The jerk bridge has terminal velocity and projected cost
\[
v_{1,\mathrm{jerk}}^\star
=\frac{15}{8T}(x_1-x_0)-\frac78v_0,
\qquad
J_T^\star=\frac{45}{T^5}\|\delta\|_2^2
=\frac{15}{T^2}C_T^\star.
\]
For fixed \(T\), these proportional costs have the same unregularized
optimal couplings. This equivalence extends to entropic OT if
the entropy coefficient is scaled with the cost:
\(\varepsilon_{\mathrm{jerk}}=(15/T^2)\varepsilon_{\mathrm{acc}}\).
Proportional costs alone do not ensure identical Sinkhorn couplings
at fixed \(\varepsilon\).

Table~\ref{tab:app_acc_vs_jerkza} compares the two variants using
acceleration matching only and the same acceleration-network
architecture. Their average \(W_2\) values are nearly identical, with
no consistent advantage across datasets. We retain the
minimum-acceleration bridge because it matches the second-order
variational formulation without imposing zero endpoint accelerations.

\begin{table}[!htbp]
\centering
\caption{Bridge-shape ablation using acceleration matching only.
\textsc{VTV-FM-Jerk-ZA} uses a minimum-jerk bridge with zero endpoint
accelerations.}
\label{tab:app_acc_vs_jerkza}
\setlength{\tabcolsep}{4.5pt}
\begin{tabular}{lcc}
\toprule
Dataset & \textsc{VTV-FM} & \textsc{VTV-FM-Jerk-ZA}\\
\midrule
Dinosaur & 0.1119 & \textbf{0.1113}\\
Swiss    & \textbf{0.1112} & 0.1120\\
Checker  & \textbf{0.1092} & 0.1096\\
Moons    & 0.0835 & \textbf{0.0829}\\
Lines    & 0.1362 & \textbf{0.1356}\\
Tree     & \textbf{0.0741} & 0.0745\\
Star     & 0.0868 & \textbf{0.0863}\\
\midrule
Average  & 0.1018 & \textbf{0.1017}\\
\bottomrule
\end{tabular}
\end{table}

\newpage

\end{document}